\documentclass{article}

\usepackage[main, final]{neurips_2026}
\usepackage{array}
\usepackage{algorithm}
\usepackage{algorithmic}
\usepackage{bbm}
\usepackage{hyperref}
\usepackage{framed}
\usepackage{microtype}
\usepackage{graphicx}
\usepackage{subcaption}
\usepackage{ulem}
\usepackage{booktabs} 

\usepackage[utf8]{inputenc} 
\usepackage[T1]{fontenc}    
\usepackage{hyperref}       
\usepackage{url}            
\usepackage{booktabs}       
\usepackage{amsfonts}       
\usepackage{nicefrac}       
\usepackage{microtype}      
\usepackage{xcolor}         

\usepackage{amsmath}
\usepackage{amssymb}
\usepackage{mathtools}
\usepackage{amsthm}
\usepackage{color,graphicx}
\usepackage{thm-restate}
\newcommand{\ouralgo}{\textsc{ARC-DuelLi}}

\theoremstyle{plain}
\newtheorem{theorem}{Theorem}[section]

\newtheorem{lemma}[theorem]{Lemma}
\newtheorem{corollary}[theorem]{Corollary}
\theoremstyle{definition}

\newtheorem{assumption}[theorem]{Assumption}
\theoremstyle{remark}
 
\usepackage[textsize=tiny]{todonotes}
\newcount\Full   
\newcommand{\iffull}[2]{\ifnum\Full=1{#1}\fi\ifnum\Full=0{#2}\fi}
\newcount\Comments  
  \newcommand{\red}[1]{{\leavevmode\color{red}{#1}}}

\newcommand{\kibitz}[2]{\ifnum\Comments=1{\textcolor{#1}{{#2}}}\fi}

\usepackage{enumitem}

\renewcommand{\cite}[1]{\citep{#1}}
\renewcommand{\emph}[1]{{\textit{#1}}}
\title{Lipschitz Dueling Bandits over Continuous Action Spaces}

\author{%
Mudit Sharma\\
Indian Institute of Technology Ropar
\And
Shweta Jain\\
Indian Institute of Technology Ropar
\And
Vaneet Aggarwal\\
Purdue University
\And
Ganesh Ghalme\\
Indian Institute of Technology Hyderabad
}

\begin{document}

\maketitle

\begin{abstract}
We initiate the study of stochastic dueling bandits over continuous action spaces with Lipschitz structure, where the learner observes only noisy pairwise comparisons. Unlike standard Lipschitz bandits, comparative feedback makes estimates across regions over a continuous action space inherently non-comparable, so classical elimination-based methods do not extend directly. Our key idea is to use a round-wise fixed reference arm, which converts pairwise comparisons into a common scalar score and enables safe recursive region elimination. Based on this idea, we propose the first algorithm for this setting, Anchored Recursive Comparisons for dueling Lipschitz bandits (\ouralgo). Under a bi-Lipschitz assumption on the transfer function, we show that \ouralgo\ achieves regret of $\tilde O\!\left(T^{\frac{d_z+1}{d_z+2}}\right)$, where $d_z$ is the zooming dimension of the near-optimal region, while using only $O(d\log T)$ space.  For the linear preference transfer function, we also prove a matching lower bound via reduction from stochastic Lipschitz bandits, establishing optimality up to logarithmic factors. Experiments show convergence to the optimal arm and improved regret over uniform-discretization dueling baselines.
\end{abstract}

\section{Introduction}

Many modern learning and optimization problems involve a continuous decision space but provide only \emph{relative} feedback. For example, when tuning a recommendation policy, ranking system, or experimental design, a user may be unable to assign reliable numerical scores to individual decisions, yet can often express a preference between two alternatives. In such settings, two structural features naturally coexist: the action space is continuous, so nearby decisions should behave similarly, and the available observations are purely pairwise comparisons rather than scalar rewards. This motivates the study of learning over continuous spaces under noisy comparative feedback.

We formalize this setting as \emph{stochastic Lipschitz Dueling Bandits}. At each round, the learner selects two actions and observes noisy feedback on which one is preferred. At the same time, the latent utility is assumed to vary smoothly over the action space, so nearby actions induce similar preference behavior. The central question is whether this geometric regularity can be exploited to obtain regret guarantees that depend on the intrinsic complexity of the near-optimal region, rather than on an apriori discretizations of the space.  

The main difficulty is that comparative feedback does not produce globally comparable estimates. Existing Lipschitz bandit algorithms \cite{Kleinberg2004,zhu2025} depend on scalar reward observations to form confidence bounds and elimination rules that are directly comparable across regions. Under dueling feedback, however, the outcome of a comparison is meaningful only relative to the opponent arm. As a result, estimates obtained in different regions are not directly comparable unless they are anchored to a common baseline. This destroys the comparability of the underlying zooming-style elimination methods and makes region-level reasoning substantially more delicate. Resolving this lack of cross-region comparability is the key technical challenge in Lipschitz dueling bandits. Prior continuous-action dueling bandit work primarily assumes convexity \cite{saha2021dueling,kumagai2017regret}, where pairwise comparisons can be converted into local first-order information and can be used in descent-based optimization frameworks. These methods are fundamentally local and rely on convex geometry; they do not yield the multiscale region-elimination machinery required in general Lipschitz spaces.  Hence, neither line of work extends directly to stochastic dueling bandits over continuous Lipschitz action spaces. To our knowledge, this is the first work on noisy stochastic dueling bandits over general Lipschitz continuous action spaces with zooming-dimension-dependent regret guarantees, leaving open the following question:

\medskip
\noindent
\emph{Can one exploit Lipschitz structure under purely comparative feedback, and if so, can one establish zooming-dimension-dependent regret guarantees?}
\medskip

We answer this question affirmatively and initiate the study of \emph{Lipschitz dueling bandits}. The key idea is to use a \emph{fixed reference arm within each round}, so that all comparisons in that round are anchored to a common baseline. Crucially, the choice of this reference arm is highly non-trivial. If it is chosen poorly, then the resulting region-level estimates can lead to poorly ranked regions and make the elimination step unsafe, potentially discarding the region containing the Condorcet winner. In this paper, we show how to choose and update the reference arm adaptively so that it induces a scalar preference field over the action space whose variation is controlled by the Lipschitz structure and whose values remain suitable for elimination. This viewpoint yields three ingredients that are absent from prior analyses: first, a uniform concentration result for region-level preference estimates relative to the common reference arm; second, a safe elimination argument showing that the region containing the Condorcet winner is never discarded; and third, a shrinking-region property proving that all surviving regions remain geometrically close to the optimum. Together, these ingredients replace the absolute-reward comparisons used in classical Lipschitz bandits by a relative but globally anchored elimination principle. Building on this analysis, we propose Anchored Recursive Comparisons for
dueling Lipschitz bandits, \ouralgo, the first algorithm for this setting that achieves regret $\tilde O\!\left(T^{\frac{d_z+1}{d_z+2}}\right)$, 
where $d_z$ is the zooming dimension of the near-optimal set. \ouralgo\ is a  depth-first elimination strategy which explores the partition tree recursively and stores  a constant amount of local information from the active recursive path. This yields an overall memory complexity of $O(\log T)$. Hence, purely comparative feedback in continuous spaces does not fundamentally increase the space complexity of zooming-style learning.

When pairwise feedback is generated via the linear transfer function, we prove a matching lower bound through a reduction to scalar-feedback Lipschitz bandits and show that every dueling bandit algorithm induces a corresponding scalar-feedback algorithm with exactly twice the regret. Therefore, existing lower bounds for stochastic Lipschitz bandits carry over directly to the dueling setting under the linear transfer model. Consequently, despite the weaker and purely relative nature of the feedback, the achievable regret of \ouralgo\ matches the best known rate for stochastic Lipschitz bandits up to logarithmic factors \cite{kleinberg2019bandits}. Finally, we complement our theory with empirical evaluation. The experiments show that \ouralgo\ converges to the optimal arm and achieves substantially lower regret than dueling-bandit baseline based on uniform discretization, illustrating the advantage of exploiting geometric structure. Overall, this work identifies Lipschitz dueling bandits as a new structured bandit model and shows that zooming-style learning is possible even when feedback is purely comparative, provided one restores cross-region comparability through an adaptive common reference.

\section{Related Work}

\paragraph{Dueling Bandits and Preference-Based Bandits.}

The dueling bandits framework proposed by \citet{yue2012k} studies sequential learning from pairwise comparisons between actions, 
rather than absolute reward feedback, with the goal of identifying or competing with the
`best performing arm' depending upon the notion of best arm that is, Condorcet winner, an action that is preferred over all the other actions with probability greater than half. Existing algorithms \cite{yue2011beat,urvoy2013generic,zoghi2014,komiyama2015regret} assume a finite or unstructured action
space and focus over arm-level elimination or confidence-based selection.

Over the continuous action space, with concave rewards, \citet{yue2009interactively} proposed the first projected gradient-based algorithm with regret guarantees of $O(T^{3/4})$. Later, \citet{kumagai2017regret} improved the regret to $O(\sqrt{T\log T})$ under the assumption of strong convexity and smoothness of the cost function. A more general cost function, $\alpha$-strong and $\beta$-smooth was later considered in \cite{jamieson2012query} which was further generalized to any $\beta$-smooth convex functions in \cite{saha2021dueling}, only for the sign-based relative feedback. A more general preference function was considered in \cite{sahadueling} with $\beta$-smooth convex functions. This paper extends this line of work to the class of more general reward functions, i.e, Lipschitz function. Some works \cite{xu2020zeroth,zhang2024comparisons} do consider non-convex optimization problem; however, they do not consider noisy feedback and hence are not in bandit feedback setup. 


\paragraph{Lipschitz and Continuum-Armed Bandits.}

A separate line of work studies bandits over continuous action spaces under
Lipschitz assumption, where the learner observes noisy scalar rewards \cite{kleinberg2019bandits,zhu2025,Kleinberg2004,feng2022lipschitz} or adversarial scalar rewards \cite{kang2023robust}. The problem was first introduced by \citet{Kleinberg2004} where first discretization based algorithm was proposed for 1-dimensional function. Later, \citet{kleinberg2008multi} presented the first zooming-based algorithm for general $d$-dimensional Lipschitz functions and also provided the lower bound of $\Omega(T^{\frac{d_z+1}{d_z+2}})$ with $d_z$ being the zooming dimension. Zooming-based algorithms essentially keep zooming in the most promising region until it reaches the near optimal region. In a follow-up work, \citet{bubeck2011x} proposed Hierarchical Optimistic Optimization (HOO) algorithm, which is easier to implement and which exploits the local properties of the reward function at its maxima only. To reduce the space complexity of zooming based algorithms, \cite{zhu2025} provided a recursive variant of zooming algorithm with space complexity of $O(\log T)$.
However, all the above methods crucially rely on access to reward feedback and do not
apply to settings with purely comparative observations.

We study stochastic dueling bandits over
continuous action spaces with Lipschitz structure. To our knowledge, this is the first work on noisy stochastic dueling bandits over general Lipschitz continuous action spaces with zooming-dimension-dependent regret guarantees.

\section{Problem Setup and Regret Definition}
\label{sec:problem-setup}
Let $\mathcal X \subset \mathbb R^d$ be a compact action space endowed with a metric $\|\cdot\|$.
Throughout the paper, we work with $\mathcal X = [0,1]^d$ under the $\ell_\infty$ metric, for which metric balls are axis-aligned cubes. This allows a particularly simple description of our recursive elimination procedure via a cube-based multiscale partition. However, the analysis is not tied to the coordinate structure of Euclidean cubes. It relies only on three abstract properties: (i) regions of controlled diameter at each scale, (ii) a hierarchical refinement into smaller regions at finer scales, and (iii) packing or covering bounds for the near-optimal set. Thus, our arguments should extend to compact doubling metric spaces endowed with a suitable hierarchical covering, though we restrict attention to $[0,1]^d$ under the $\ell_\infty$ metric for clarity. Further, we assume that the unknown reward function $f:\mathcal X \to \mathbb R$
satisfies the following regularity condition.
\begin{assumption}[Lipschitz reward]
\label{ass:lipschitz-f}
The function $f$ is $1$-Lipschitz with respect to $\|\cdot\|$, i.e.,
\[
|f(x)-f(y)| \le \|x-y\|,
\qquad \forall x,y\in\mathcal X.
\]
\end{assumption}

Unlike the Lipschitz bandit, the learner does not observe the values of $f$ directly in a dueling feedback model. Instead, based on past observations, at each time $t$, the learner selects two arms denoted by $(x_t,y_t) \in \mathcal{X}^2$ and only observes a noisy binary (duel) outcome $o_t(x_t,y_t)$ of the stochastic duel. The learner operates over a horizon of $T$ duels.
The duel outcome is a Bernoulli random variable
\[
o_t(x_t,y_t) \sim \mathrm{Ber}\!\left(P(x_t\succ y_t)\right).
\]

Here, $x_t \succ y_t$ denote the event that arm $x_t$ is preferred over $y_t$. We connect the pairwise preference probability $P(x_t \succ y_t)$ through a transfer function $\rho: [-1,1] \rightarrow [0,1]$ as $P(x_t \succ y_t) = \rho(f(x_t)-f(y_t))$. Note that since \(f\) is 1-Lipschitz on \([0,1]^d\) under \(\ell_\infty\), we have
\(|f(x)-f(y)|\le 1\) for all \(x,y\in\mathcal X\). Hence the transfer
function only needs to be specified on \([-1,1]\). Different examples of transfer functions include : (a) Bradley--Terry transfer (Logistic) \cite{bengs2021preference}, (b) Linear transfer function \cite{ailon2014reducing}, and (c) probit transfer function \cite{chen2017dueling}. This paper makes the following assumption on the transfer function. 
\begin{assumption}
\label{ass:transferlipschitz}
The transfer function \(\rho:[-1,1]\to[0,1]\) satisfies $\rho(0)=\frac12, \rho(-z)=1-\rho(z)$, and there exist constants \(0<\lambda\le \Gamma<\infty\) such that for all
\(-1\le z_2\le z_1\le 1\),
\[
\lambda (z_1-z_2)
\le
\rho(z_1)-\rho(z_2)
\le
\Gamma (z_1-z_2).
\]
\end{assumption}
It is easy to show that Bradley-Terry, Probit, and Linear transfer functions satisfy these assumptions and are  bi-Lipschitz with $(\Gamma,\lambda) = \left(\frac{1}{4},\frac{e}{(1+e)^2}\right)$ (Lemma \ref{lem:sigmoid-lipschitz}), $(\Gamma, \lambda) = \left(\frac{1}{\sqrt{2\pi}},\frac{1}{\sqrt{2\pi}}e^{-1/2}\right)$ (Lemma \ref{lem:probit-lipschitz}), and $(\Gamma,\lambda) = \left(\frac{1}{2},\frac12\right)$ (Lemma \ref{lem:linear-lipschitz}) respectively. 


An arm $x^*$ is called the Condorcet winner if $P(x^* \succ y) \ge \frac{1}{2}$ for all $y \in \mathcal{X}$.
Under Assumption~\ref{ass:lipschitz-f}, the function $f$ is continuous. Since $\mathcal{X}$ is compact, there exists $x^* \in \arg\max_{x \in \mathcal{X}} f(x)$.
Pairwise preferences are induced via a transfer function $\rho : [-1,1] \to [0,1]$ such that $P(x \succ y) = \rho(f(x) - f(y))$, where $\rho$ is non-decreasing and satisfies $\rho(0) = \frac{1}{2}$. Therefore, for any $y \in \mathcal{X}$,
\[
f(x^*) \ge f(y) \;\Rightarrow\; P(x^* \succ y) \ge \frac{1}{2}.
\]
Thus, a Condorcet winner always exists.

\subsection{Regret in Lipschitz Dueling Bandits}
For any pair of arms $(x,y)$, define the (pairwise) preference gap $ 
\Delta(x,y)
:=
P(x\succ y) - \tfrac12 .$ 
Note that $\Delta(x^\star,x) \ge 0$ for all $x\neq x^\star$ also under Assumption \ref{ass:transferlipschitz}, $\Delta(x,y) = -\Delta(y,x)$. Further, let $\kappa = \frac\Gamma\lambda$. We next show that under Assumption \ref{ass:transferlipschitz}, preference gap $\Delta$ function satisfy relaxed stochastic triangular inequality and reverse difference inequality.
\begin{restatable}{lemma}{LemmaZero}
\label{lemma:sti}
Under Assumption \ref{ass:transferlipschitz}, the preference gap \(\Delta\) satisfies:
\begin{enumerate}
\item if \(f(x)\ge f(z)\ge f(y)\), then $\Delta(x,y)\le \kappa(\Delta(x,z)+\Delta(z,y))$,
\item if \(f(x)\ge f(y)\ge f(z)\), then $\Delta(x,y)\le \kappa(\Delta(x,z)-\Delta(y,z))$.
\end{enumerate}
\end{restatable}
The proof uses simple inequalities arising from bi-Lipschitz property and is provided in Appendix \ref{lemma:sti-transfer}.

The instantaneous regret $r_t$ incurred at time $t$ is defined as the sum
of preference gaps between the optimal arm $x^\star$ and the two arms
played in the duel i.e., 
$r_t
 :=
\Delta(x^\star,x_t) + \Delta(x^\star,y_t).$
The cumulative regret after $T$ duels is
\begin{equation}
\label{eq:regret-def}
R(T)
:=
\sum_{t=1}^T r_t
=
\sum_{t=1}^T
\Big(
\Delta(x^\star,x_t)
+
\Delta(x^\star,y_t)
\Big).
\end{equation}

This regret notion is standard in stochastic dueling bandits and
measures how much worse the learner’s comparisons are relative to
always comparing the Condorcet winner against itself \cite{yue2012k} . Under the  convex loss function, \citet{kumagai2017regret} presents a dueling bandit algorithm for achieving regret of $O(\sqrt{T\log T})$. However, under a Lipschitz reward function the regret is often dependent on zooming dimension \cite{zhu2025, Kleinberg2004} which  captures the intrinsic dimensionality of the near-optimal
region and governs the achievable regret rates.

For $\varepsilon>0$, define the $\varepsilon$-near-optimal set 
$\mathcal S(\varepsilon)
:=
\{x\in\mathcal X : \Delta(x^\star,x)\le \varepsilon\}.
$ 
We assume that $\mathcal S(\varepsilon)$ has \emph{zooming dimension}
$d_z$, meaning that there exists $C_z>0$ such that for all $\varepsilon>0$,
the set $\mathcal S(\varepsilon)$ can be covered by at most
$C_z \varepsilon^{-d_z}$ balls of radius $\varepsilon$. 

Thus, the goal in Lipschitz dueling bandit setting is to design a learning algorithm that minimizes the expected
regret $\mathbb E[R(T)]$ while only observing noisy dueling feedback,
and to characterize how the regret scales with the time horizon $T$ and the geometric complexity of the near-optimal region.

\section{\ouralgo: Anchored Recursive Comparison for Dueling Lipschitz}
\label{sec:algorithm}

\begin{figure*}[t]
\centering

\begin{minipage}[t]{0.49\textwidth}
\begin{algorithm}[H]
\caption{Anchored Recursive Comparison for Dueling Lipschitz Bandits (\ouralgo)}
\label{alg:logli}
\begin{algorithmic}[1]
\REQUIRE Action space $\mathcal X=[0,1]^d$, horizon $T$, confidence $\delta$, $\Gamma$
\STATE Choose number of rounds $B$
\STATE Set radius sequence $r_m = 2^{-m+1}$ for $m=1,\dots,B+1$
\STATE Set sample sizes $n_m = \frac{16\log(T/\delta)}{r_m^2}$
\STATE \textbf{Initialize:}  $t \leftarrow 0$,   $X^0 \sim \mathrm{Unif}(\mathcal X)$,
$\tilde\Delta^0=0$
\FOR{$m=1$ to $B$}
    \STATE Set $\Delta_{\max}^m=-\infty$, $C_{\max}^m=\emptyset$
    \STATE Partition $\mathcal X$ into $(1/r_1)^d$ cubes of edge length $r_1$
    \FOR{each cube $C$}
        \STATE \textsc{RoundFunc} $(t,m,1,C, \Delta_{\max}^m,C_{\max}^m,
        \tilde\Delta^{m-1},X^{m-1})$
    \ENDFOR
    \IF{$C_{\max}^m \neq \emptyset$ and $t+n_m \le T$}
        \STATE Choose $X^m$ uniformly from the sampled arms in $C_{\max}^m$
        \STATE Sample $n_m$ arms uniformly from $C_{\max}^m$ and duel each with $X^m$
        \STATE Compute empirical gap $\widehat{\Delta}_m^m(C_{\max}^m,X^m)$
        \STATE $\tilde\Delta^m \gets \widehat{\Delta}_m^m(C_{\max}^m,X^m)$, $t\gets t+n_m$
    \ELSE
        \STATE Return arm $X^{m-1}$
    \ENDIF
\ENDFOR
\STATE \textbf{Cleanup:} play arm $X^B$ or $X^{m-1}$ for all remaining time steps
\end{algorithmic}
\end{algorithm}
\end{minipage}
\hfill
\begin{minipage}[t]{0.47\textwidth}
\begin{algorithm}[H]
\caption{Recursive cube exploration (\textsc{RoundFunc})}
\label{alg:roundfunc}
\begin{algorithmic}[1]
\REQUIRE Time horizon $T$, current time counter $t$,
round index $m$, current depth $h$, cube $C$,
current-round maximum $\Delta_{\max}^m$, best cube $C_{\max}^m$, previous-round estimate
$\tilde\Delta^{m-1}$, and reference arm $X^{m-1}$, $\Gamma$.
\IF{$t + n_h > T$}
    \STATE \textbf{return}
\ENDIF
\STATE Sample $n_h$ arms uniformly from $C$ and duel each with $X^{m-1}$
\STATE Compute empirical gap $\widehat{\Delta}_h^m(C,X^{m-1})$
\STATE $t \gets t + n_h$
\IF{$h = m$}
    \IF{$\widehat{\Delta}_h^m(C,X^{m-1}) > \Delta_{\max}^m$}
        \STATE $\Delta_{\max}^m \gets \widehat{\Delta}_h^m(C,X^{m-1})$ , $C_{\max}^m \gets C$
    \ENDIF
\ELSE
    \IF{$\tilde\Delta^{m-1} - \widehat{\Delta}_h^m(C,X^{m-1}) \le 2\, r_h(1+\Gamma)$}
        \STATE Partition $C$ into $(r_h/r_{h+1})^d$ subcubes
        \FOR{each subcube $C'$}
            \STATE \textsc{RoundFunc}$(T,t,m,h+1,C',
\Delta_{\max}^m,C_{\max}^m,
\tilde\Delta^{m-1},X^{m-1})$
        \ENDFOR
    \ELSE
        \STATE \textbf{Eliminate} cube $C$ (terminate recursion)
    \ENDIF
\ENDIF
\end{algorithmic}
\end{algorithm}
\end{minipage}

\end{figure*}

We present \ouralgo, Anchored Recursive Comparison for Dueling Lipschitz  in
Algorithm~\ref{alg:logli} (main procedure) and
Algorithm~\ref{alg:roundfunc} (recursive subroutine, \textsc{RoundFunc}).
\ouralgo\ combines round-based exploration with recursive cube elimination
to progressively localize the Condorcet winner. For each depth $h\ge1$, we define a resolution
$r_h = 2^{-h+1}$.
At this resolution, $\mathcal X$ is partitioned into axis-aligned hypercubes
(``cubes'') of edge length $r_h$, each representing a region of diameter at most
$r_h$. A cube at depth $h+1$ is obtained by equally partitioning a cube at depth
$h$ into $2^d$ subcubes, yielding a hierarchical multiscale partition of
$\mathcal X$.


The algorithm proceeds in rounds $m=1,\dots,B$, where $B = \left\lceil \frac{\log_2 T}{d_z+2} \right\rceil + 1$.
The parameter $B$ balances exploration depth and the resolution of sub optimality within cubes. 
Its choice is justified by the regret guarantee in Theorem~\ref{thm:regret-theorem}. At the start of round $m$, the algorithm maintains a \emph{reference arm} $X^{m-1}$ and the best empirical preference gap $\tilde{\Delta}^{m-1}$ from the previous round. All comparisons in round $m$ are made against $X^{m-1}$, ensuring that empirical gap estimates across cubes are directly comparable.


During round $m$, the action space is first partitioned into cubes at depth $1$ (line~7). Each cube is then explored by invoking the recursive subroutine \textsc{RoundFunc} starting at depth $h=1$ (lines~8--9). The subroutine processes a cube $C$ at depth $h$ as follows. If sufficient time remains (Algorithm~\ref{alg:roundfunc}, lines~1--2), it samples $n_h = \frac{16\log(T/\delta)}{r_h^2}$ arms uniformly from $C$, and each is dueled against the reference arm $X^{m-1}$ (lines~3--4). From these samples, it computes the empirical preference gap $\widehat{\Delta}_h^m(C,X^{m-1})$ (line~5) using equation \ref{eq:phat-def}. The sample size $n_h$ guarantees that $\widehat{\Delta}_h^m(C,X^{m-1})$ closely approximates $\Delta(x, X^{m-1})$ uniformly for all $x \in C$ at depth $h$ (Lemma~\ref{lemma:bound-prob-diff}).


If the maximum depth $h = m$ is reached, the cube is not refined further. 
Instead, the algorithm sets $\Delta^m_{\max}$ to the maximum of $\widehat{\Delta}_h^m(C,X^{m-1})$ over all terminal cubes in round $m$, and stores a corresponding maximizer $C_{\max}^m$ (lines~8--9). 
If such a cube exists, the next reference arm is selected from $C_{\max}^m$ (lines~11--12), and $\tilde{\Delta}^m$ is refined by drawing $n_m$ additional samples against the reference arm (lines~13--15). If, at any point during round $m$, insufficient time remains to complete the next required sampling block, the recursive exploration is terminated and the main procedure returns $X^{m-1}$ (lines~17), namely, the reference arm obtained from the most recently completed round. Thus, an incompletely explored round is never used to update the reference arm. we later prove that even if such case exist then even our regret bounds remain the same. The choice of the reference arm is crucial, as all future elimination decisions are made relative to it.


If $h<m$, the algorithm decides whether to eliminate or refine the cube. 
Specifically, cube $C$ is eliminated if $\tilde{\Delta}^{m-1} - \widehat{\Delta}_h^m(C) > 2\,r_h(1+\Gamma)$ (line~12), which, by concentration and Lipschitz arguments, implies that all arms in $C$ are suboptimal with high probability (Lemma~\ref{lem:optimal-not-eliminated}). Otherwise, $C$ is partitioned into $2^d$ sub-cubes of edge length $r_{h+1}$ and explored recursively (lines~13--14).

The algorithm takes the transfer-function Lipschitz constant $\Gamma$ as an input. In practice, one may use an estimate or upper bound of $\Gamma$ which affects the elimination rule.  If $\Gamma$ is underestimated, the elimination threshold becomes too aggressive, and the algorithm may incorrectly eliminate the cube containing the Condorcet winner. On the other hand, if $\Gamma$ is substantially overestimated, the threshold becomes overly conservative: more suboptimal cubes survive, exploration becomes less focused, and the regret may increase. Thus, the choice of $\Gamma$ induces a natural safety--efficiency tradeoff. 

Although written recursively, the procedure is implemented as an in-place depth-first traversal. At any time, it stores only the current cube, current depth, current reference, current best candidate, and a constant number of counters and empirical estimates.
Information is discarded once recursion on a cube terminates, yielding a total memory usage of $O(\log T)$ due to the maximum depth $B=O(\log T)$. Across rounds, the reference arm is progressively localized near the Condorcet
winner with high probability.
After $M \leq B$ rounds where M is the last completed round, Lemma~\ref{lem:shrinking-region-dueling} ensures that all surviving arms lie within $O(r_{M-1})$ preference gap of the Condorcet winner. 
The algorithm then enters a cleanup phase, repeatedly playing the final reference arm $X^B$ or $X^{m-1}$ until the horizon $T$ is exhausted (Algorithm~\ref{alg:logli}, line~20).


\section{Regret Analysis for \ouralgo} 
First lemma shows that with $n_h$ number of duels in each round $m$ at depth $h$, the empirical preference gap for a cube $C$ closely approximate the preference gap for an arbitrary arm in $C$ with any reference arm $Y$. The proof essentially uses Hoeffding's inequality and Lipschitz property of the transfer function. The detailed proof is provided in Appendix \ref{app:lemone}. 
\begin{restatable}{lemma}{LemOne}[Uniform concentration for cube-level preference estimates]
\label{lemma:bound-prob-diff}
Fix a round $m$, depth $h$, and a cube $C\in\mathcal{A}_h^m$ of diameter at most $r_h$.
Let $Y\in\mathcal{X}$ be
the reference arm fixed before that sampling block begins and let
$\{x_i\}_{i=1}^{n_h}$ be i.i.d.\ samples drawn uniformly from $C$.
Define the empirical estimator $\widehat{\Delta}_h^m(C,Y)$ and event $\mathcal{E}$ as:
\begin{align}
\label{eq:phat-def}
\widehat{\Delta}_h^m(C,Y)
\;:=\;
\frac{1}{n_h}\sum_{i=1}^{n_h} o(x_i,Y) - \frac{1}{2}, \\
\label{eq:event-E}
\mathcal{E}
:=
\Bigg\{
\big|
\Delta(x,Y)
-
\widehat{\Delta}_h^m(C,Y)
\big|
\le\;
\sqrt{\frac{16\log(T/\delta)}{n_h}}
+
\Gamma r_h
\Bigg\},
\end{align}
for all
$1\le h\le m\le B$,
all $C\in\mathcal{A}_h^m$,
and all $x\in C$. Then, $\mathbb{P}(\mathcal{E}) \;\ge\; 1-2\delta.$ Here, $o(x_i,Y) \sim Ber(P(x_i\succ Y))$.
\end{restatable}
\paragraph{Proof sketch.}
For a fixed tuple $(m,h,C)$, the error 
$\big|\Delta(x,Y)-\widehat{\Delta}_h^m(C,Y)\big|$
is decomposed into a stochastic term and a bias term for an arbitrary arm $x$. Conditioned on the reference arm $Y$ and the cube $C$, the duel outcomes against $Y$ are i.i.d.\ Bernoulli, so Hoeffding's inequality gives concentration of $\widehat{\Delta}_h^m(C,Y)$ around its expectation at rate $\sqrt{16\log(T/\delta)/n_h}$. The remaining bias comes from replacing the gap of the fixed arm $x$ by the average gap over sampled points in $C$ and can be upper bounded  $\Gamma r_h$ through Lipschitz continuity of $f$ and  $\rho$. Combining the two bounds yields the stated deviation inequality, and a union bound over all rounds and active cubes gives the event $\mathcal{E}$ with probability at least $1-2\delta$.

The next lemma shows that when elimination occurs at line 12 of recursive subroutine, then the eliminated cube does not contain an optimal arm with high probability. The proof essentially uses Lemma \ref{lemma:bound-prob-diff}, the optimality of Condorcet arm, and the fact that comparisons are made against the same reference arm. The detailed proof is provided in Appendix \ref{app:lemtwo} 
\begin{restatable}{lemma}{LemTwo}[Optimal arm is never eliminated]
\label{lem:optimal-not-eliminated} Suppose that the algorithm completes $M\leq B$ rounds.
Under the event $\mathcal{E}$ defined in
Lemma~\ref{lemma:bound-prob-diff}, an optimal arm
$x^\star \in \arg\max_{x\in\mathcal{X}} f(x)$
is not eliminated in any of the completed rounds $m\in[M]$.
Consequently, the cube containing $x^\star$ never gets eliminated.

\end{restatable}
\paragraph{Proof sketch.}
Suppose that the algorithm completes $M\leq B$ rounds. Fix any completed
round $m\in[M]$ and any depth $h<m$, and let $(C_h^m)^\star$ denote the
depth-$h$ cube containing the optimal arm $x^\star$. The key point is that, after the benchmark $\tilde{\Delta}^{m-1}$ is redefined in Line 15 using the newly chosen reference arm $X^{m-1}$, both $\tilde{\Delta}^{m-1}$ and the current-round estimate $\widehat{\Delta}_h^m\!\big((C_h^m)^\star,X^{m-1}\big)$ are evaluated against the same reference arm. This makes them directly comparable. Under the high-probability event $\mathcal E$, Lemma~\ref{lemma:bound-prob-diff} controls both quantities up to the usual concentration and Lipschitz error. Since $x^\star$ maximizes $f$ and the transfer function is monotone, its true preference gap against $X^{m-1}$ is at least as large as that of any arm in the cube from which $X^{m-1}$ was selected. Therefore, the difference $\tilde{\Delta}^{m-1}-\widehat{\Delta}_h^m\!\big((C_h^m)^\star,X^{m-1}\big)$ is at most the elimination threshold $2r_h(1+\Gamma)$. Hence the cube containing $x^\star$ cannot be eliminated. Since this holds at every depth and every completed round $m\le M$, the optimal arm is never discarded during the $M$ completed rounds. If the algorithm stops during round $M+1$, all partial updates from that incomplete round are discarded, and the algorithm returns the reference arm from previous completed round i.e. $X^{m-1}$.

The next lemma shows that as we keep shrinking the cube, the gap between optimal arm and any arm in that cube decreases. For this, we use Lemma \ref{lemma:sti} along with the Lemma \ref{lemma:bound-prob-diff}. The detailed proof is provided in Appendix~\ref{app:lemthree}.

\begin{restatable}{lemma}{LemThree}[Shrinking-region property for dueling bandits]
\label{lem:shrinking-region-dueling}
Suppose that the algorithm completes $M\le B$ rounds.
Under the event $\mathcal{E}$, for any
$1 \le h \le m \le M$,
any cube $C \in \mathcal{A}_h^m$,
and any arm $x \in C$,
the Condorcet gap satisfies
$\Delta(x^\star,x) \le \kappa(2\kappa + 4)(1+\Gamma)r_{h-1}$,
with $\kappa = \frac{\Gamma}{\lambda}$.
\end{restatable}
\paragraph{Proof sketch.}
Fix a completed round $m\le M$, a depth $h\le m$, a surviving cube $C\in \mathcal A_h^m$, and an arm $x\in C$. For $h=1$, the claim is immediate since preference gaps are bounded. For $h\ge 2$, let $C_{\mathrm{par}}$ denote the parent of $C$ at depth $h-1$. The proof compares $x$ to the optimal arm $x^\star$ through the common reference arm $X^{m-1}$ and applies Lemma \ref{lemma:sti} to write $\Delta(x^\star,x)
\le
\frac\Gamma\lambda\left(\Delta(x^\star,X^{m-1})-\Delta(x,X^{m-1})\right)$.
The first term is controlled by showing that the newly chosen reference
arm $X^{m-1}$ comes from a cube whose empirical gap against the previous
reference $X^{m-2}$ is at least as large as that of the optimal cube;
under the concentration event $\mathcal E$, this implies that
$\Delta(x^\star,X^{m-1})$ is at most
$O(r_{m-1}(1+\Gamma))$. The second term is controlled by the fact that
the parent cube $C_{\mathrm{par}}$ survives elimination in round $m$, so
its empirical estimate against $X^{m-1}$ cannot be much smaller than the
benchmark $\tilde{\Delta}^{m-1}$. Using
Lemma~\ref{lemma:bound-prob-diff} to pass from empirical cube estimates
to true preference gaps then gives a lower bound on
$\Delta(x,X^{m-1})$. Combining the two bounds yields $\Delta(x^\star,x)\le\kappa(2\kappa+4)(1+\Gamma)r_{h-1}$ showing that every surviving cube in each completed round $m\le M$ lies
geometrically close to the optimum.

We now prove our main theorem that bounds the regret for \ouralgo\ by combining the above lemmas and finding the optimal value of $B$. The detailed proof is given in Appendix~\ref{app:thmone}.
\begin{restatable}{theorem}{ThmOne}[Regret bound]
\label{thm:regret-theorem}
Assume that the near-optimal set
$\mathcal S(\varepsilon)
:=
\{x\in\mathcal X : \Delta(x^\star,x)\le \varepsilon\}$ 
has zooming dimension $d_z$ under the metric $\|\cdot\|$. 
Then, under the event $\mathcal E$, the regret of the \ouralgo\ satisfies
\[
R(T)
=
\tilde O\!\left(
T^{\frac{d_z+1}{d_z+2}}
\right),
\]
by setting $B=\left\lceil\frac{\log_2 T}{d_z+2}\right\rceil + 1$. The bound holds even if the algorithm completes only $M<B$ rounds and
stops during round $M+1$ due to insufficient remaining time.
Here, $\tilde O(\cdot)$ hides logarithmic factors in $T$ and $1/\delta$.
\end{restatable}
 \paragraph{Proof sketch} 

Let $M\le B$ denote the number of rounds completed by the algorithm. We decompose the regret into the contribution accumulated during the completed rounds, the possible incomplete round $M+1$, and the remaining cleanup phase. For a fixed completed round $m\le M$, every duel compares the round reference arm $X^{m-1}$ with an arm drawn from some surviving cube at depth $h$. By Lemma~\ref{lem:shrinking-region-dueling}, both the played arm and the reference arm are $O(r_{h-1}(1+\Gamma))$-optimal, so each duel at depth $h$ incurs regret at most $O(r_h(1+\Gamma))$. Since each surviving cube is sampled $n_h=\Theta(\log(T/\delta)/r_h^2)$ times, the regret contribution at depth $h$ is controlled by $|\mathcal A_h^m|\cdot\widetilde O(1/r_h)$. The zooming-dimension assumption bounds the number of surviving depth-$h$ cubes by $|\mathcal A_h^m|\le C_zr_h^{-d_z}$, yielding a round-wise regret of order
$\widetilde O\!\left(\sum_{h=1}^m2^{(h-1)(d_z+1)}\right)$. If $M=B$, summing over all rounds gives an exploration term of order $\widetilde O(2^{(B-1)(d_z+1)})$. After round $B$, Lemma~\ref{lem:shrinking-region-dueling} implies that the final reference arm lies within $O(r_{B-1}(1+\Gamma))$ of the optimum, so the cleanup phase contributes at most $O(2^{-B}T)$. If $M<B$, the algorithm stops during round $M+1$ because insufficient time remains to complete the next sampling block. The regret incurred during the completed blocks of this incomplete round is bounded by the same argument as above. Moreover, if $\tau$ is the remaining number of time steps, the stopping condition gives $\tau<n_h\le n_{M+1}$. Since the algorithm plays $X^M$ during these steps, the cleanup regret is at most
$
O(r_{M-1}n_{M+1})
=
\widetilde O(2^M)
\le
\widetilde O\!\left(2^{M(d_z+1)}\right).
$
Thus, early stopping does not increase the regret order. Choosing
$B=\left\lceil\frac{\log_2 T}{d_z+2}\right\rceil+1$
balances the exploration and cleanup terms when all rounds are completed and, in both cases, yields
$
R(T)
=
\widetilde O\!\left(
T^{\frac{d_z+1}{d_z+2}}
\right).
$

Next, we provide space complexity analysis for \ouralgo\  which closely follow the  existing depth-based elimination algorithm\cite{zhu2025}. Detailed proof is provided in Appendix~\ref{app:space}.
\begin{restatable}{lemma}{LemmaFour}[Space complexity]
\label{lemma:Space-analysis}
\ouralgo\ requires at most
$O(d\log T)$ bits of memory at any time.
\end{restatable}

\section{Lower Bound}
In lower bound, we prove that given any dueling algorithm, there exists a scalar-feedback bandit algorithm whose regret is twice of that of dueling algorithm. 
Since our model class includes the linear transfer function, we provide the lower bound for this subclass. Let $\mathcal{H}$ be any class of stochastic Lipschitz bandit instances on
$\mathcal{X}$, where each instance is specified by a mean reward function
$\mu:\mathcal{X}\to[0,1]$, and pulling arm $x$ returns an independent Bernoulli
reward with mean $\mu(x)$. For each $\mu\in\mathcal{H}$, define the induced dueling instance
$\Phi(\mu)$ on the same action space by
\[
P_\mu(x\succ y)=\frac12\bigl(1+\mu(x)-\mu(y)\bigr), \qquad x,y\in\mathcal{X}.
\]
This is exactly the dueling model with linear transfer function. Then we have the following theorem:

\begin{restatable}{theorem}{ThmTwo}[Reduction]
\label{thm:alg-reduction}
Let $A$ be any dueling bandit algorithm with linear transfer function operating for $T$ rounds. Then there exists a scalar-feedback bandit algorithm $B_A$ operating for $2T$
rounds such that, for every instance $\mu$,
$\mathbb E_\mu\!\left[R^{\mathrm{sc}}_{B_A}(2T)\right]
=
2\,\mathbb E_\mu\!\left[R^{\mathrm{duel}}_{A}(T)\right].$ Here, $R^{\mathrm{sc}}_{B_A}(2T)$ represents regret by algorithm $B_A$ for scalar-feedback bandit algorithm and $R^{\mathrm{duel}}_{A}$ represent regret by algorithm $A$ under duel-feedback bandit algorithm.
\end{restatable}
\paragraph{Proof sketch.}
Given any dueling bandit algorithm $A$, we construct a scalar-feedback algorithm
$B_A$ that simulates $A$ using two scalar pulls per duel. Whenever $A$ queries
$(x_t,y_t)$, the algorithm $B_A$ pulls both arms once and converts the two scalar
observations into a synthetic duel outcome by declaring the larger reward the
winner and breaking ties uniformly at random. Under the linear transfer function, this
synthetic comparison has exactly the same distribution as a true duel. Moreover, the
scalar regret incurred by the two pulls is exactly twice the dueling regret at
that round. Summing over all rounds and taking expectations yields the claim. 

Thus, the lower bound on $B_A$ is thus applicable to $A$. Noting that for every $A$, there exists a corresponding $B_A$, we get the desired Corollary: 
\begin{corollary}
\label{cor:lb}
If the corresponding stochastic Lipschitz bandit class satisfies the lower
bound
$\inf_B \sup_{\mu\in\mathcal H}
\mathbb E_\mu[R_B^{\mathrm{sc}}(T)]
=
\Omega\!\left(T^{\frac{d_z+1}{d_z+2}}\right)$, 
then the lower bound on dueling bandit class satisfies
$\inf_A \sup_{\nu\in\Phi(\mathcal H)}
\mathbb E_\nu[R_A^{\mathrm{duel}}(T)]
=
\Omega\!\left(T^{\frac{d_z+1}{d_z+2}}\right)$.
\end{corollary}

\section{Experiments}

We evaluate the proposed algorithm on a synthetic continuous dueling bandit instance to study four key questions: (i) whether the sequence of reference arms converges toward the optimal action, (ii) whether the preference gap between the optimal arm and the reference arm decreases across rounds as predicted by our theory, (iii) how cumulative regret scales with the time horizon, and (iv) how the proposed adaptive zooming strategy compares with a uniform discretization baseline.

We consider the continuous action space $\mathcal{X}=[0,1]^2$. The utility of an arm $x \in \mathcal{X}$ is generated according to $\mu(x)=1-\frac{2}{3}\|x-x_1\|-\frac{1}{3}\|x-x_2\|$, where $x_1=(0.3,0.1),  x_2=(0.9,0.9)$. This reward landscape induces a unique maximizer at $x^\star=(0.3,0.1)$. Pairwise comparisons are stochastic and follow the logistic transfer model.
Thus, arms with higher utility are more likely to win pairwise duels, with confidence increasing as the utility gap grows.

Unless stated otherwise, we set the zooming dimension to $d_z=0$. For the convergence experiments, we run the algorithm for horizon $T=10^8$ over $50$ independent trials. To test robustness, we initialize the first reference arm at the deliberately suboptimal point $(0.9,0.9)$, which lies far from the optimum. All reported results are averaged over repeated runs. 

\begin{figure*}[h]
    \centering
    
    \begin{subfigure}[t]{0.32\textwidth}
        \centering
        \includegraphics[width=\linewidth]{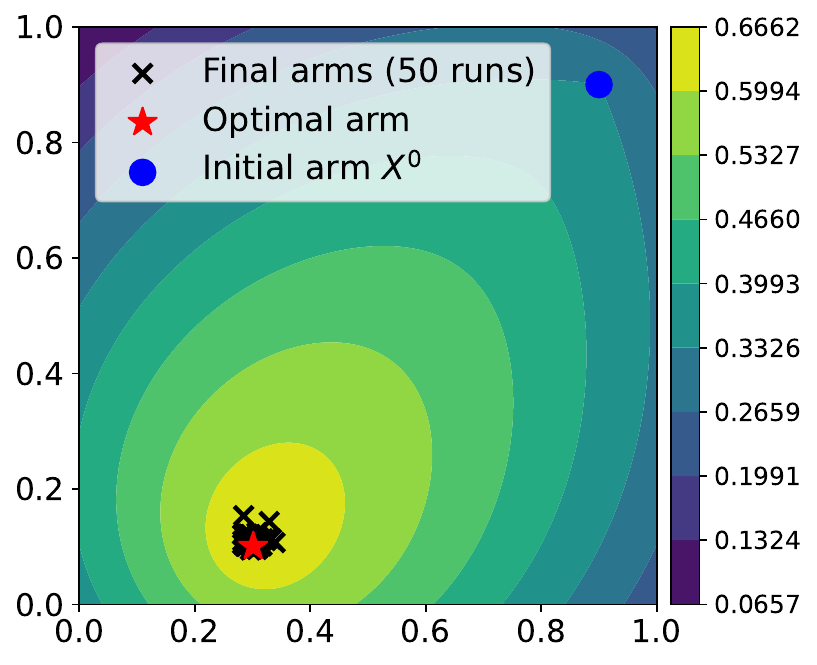}
        \caption{Reference arm convergence}
        \label{subfig:optconvergence}
    \end{subfigure}
    \hfill
    \begin{subfigure}[t]{0.32\textwidth}
        \centering
        \includegraphics[width=\linewidth]{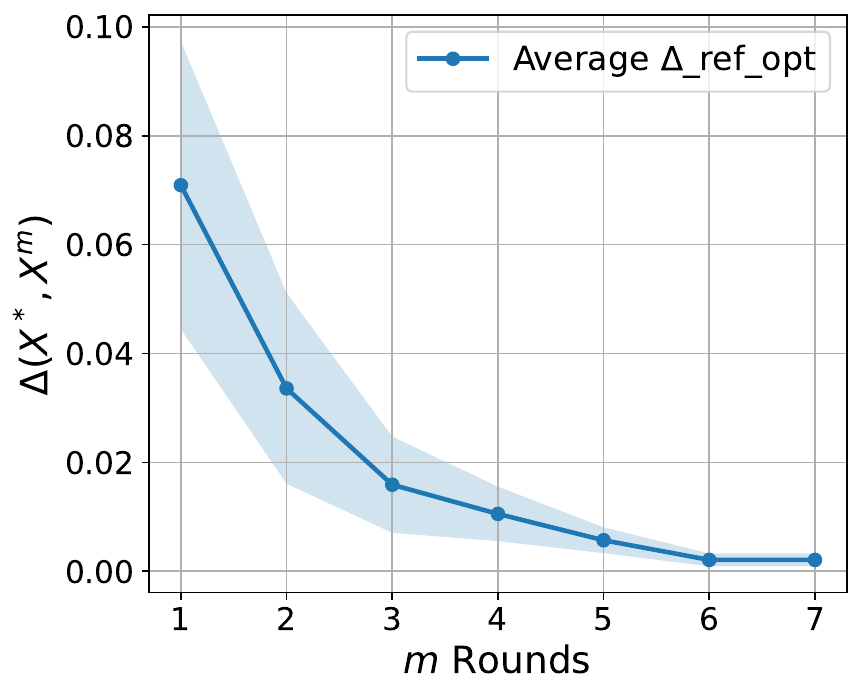}
        \caption{Convergence of $\Delta(x^*,X^m)$}
        \label{subfig:deltaconvergence}
    \end{subfigure}
    \hfill
    \begin{subfigure}[t]{0.32\textwidth}
        \centering
        \includegraphics[width=\linewidth]{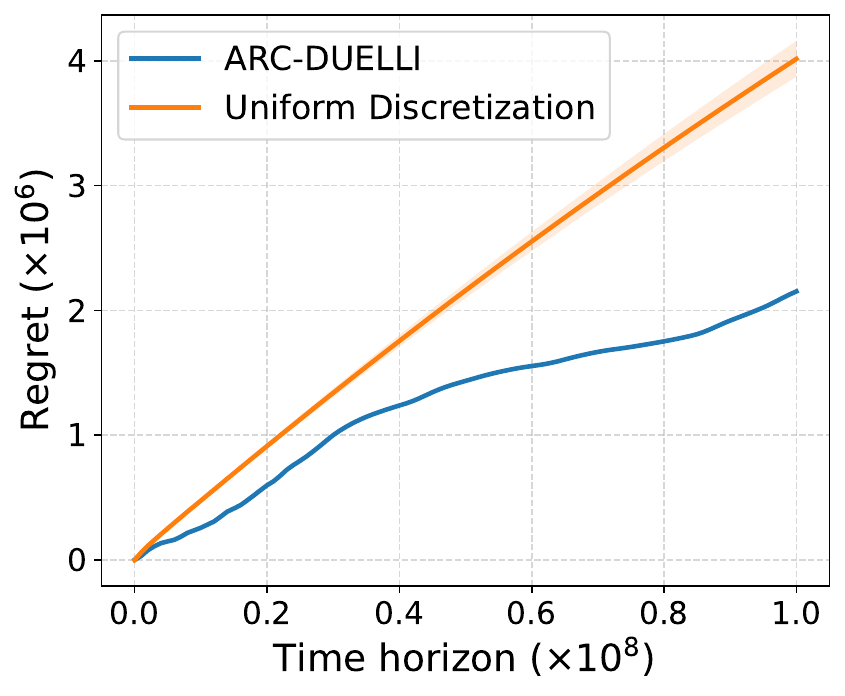}
        \caption{Regret Comparison}
        \label{subfig:regret}
    \end{subfigure}

\caption{Performance of \ouralgo\ over 50 runs. 
(a) Final reference arms concentrate near the optimal arm. 
(b) Decay of the gap $\Delta(x^\star, X^m)$ across rounds, showing convergence of the reference arm. 
(c) Cumulative regret compared against a uniform discretization baseline (Interleaved Filter).The solid curves denote averages over runs, and the shaded regions represent $\pm 1$ standard deviation.}
    \label{fig:main figure}
\end{figure*}

Figure~\ref{subfig:optconvergence} shows the trajectory of the reference arms. 
Starting from an unfavorable initialization $(0.9,0.9)$, the algorithm consistently drives the reference toward the neighborhood of the optimal action across all runs. 
To validate our theory, we track $\Delta(x^\star, X^m)$, where $X^m$ is the reference arm in round $m$. 
Figure~\ref{subfig:deltaconvergence} shows that this gap decreases steadily with $m$, approaching zero, which indicates improving reference quality and supports Lemma~\ref{lem:shrinking-region-dueling}. For fixed $T=10^8$, the gap stabilizes after several rounds and does not vanish further, as the exploration depth is bounded and zooming eventually stops. 


Figure~\ref{subfig:regret} compares our method with a baseline that uniformly discretizes the action space into a grid with cube radius $r = 2^{-m+1}$, matching the resolution used in \ouralgo\ for the same depth parameter $m$. Each cube is treated as an arm, and the Interleaved Filter algorithm~\cite{yue2012k} is applied on this discretized space. Our zooming-based method significantly outperforms this baseline. Evaluation of \ouralgo\ under different settings, including random initialization and reward functions beyond $[0,1]$, and weak regret are provided in Appendix~\ref{sec:Additional_experiment}.

Overall, the experiments corroborate the theoretical predictions: the reference arms converge to the optimum, the preference gap decays across rounds, cumulative regret is sublinear, and adaptive zooming significantly outperforms uniform discretization in continuous dueling settings.
\section{Discussion and Future Work}
We study Lipschitz dueling bandits over continuous action spaces and propose \ouralgo, the first algorithm that provably exploits geometric structure under purely comparative feedback. We show a regret bound of  $\tilde O\!\left(T^{\frac{d_z+1}{d_z+2}}\right)$, matching the best known rates for stochastic Lipschitz bandits up to logarithmic factors, while using only logarithmic memory.
These results show that purely comparative (dueling) feedback is sufficient to achieve regret guarantees that depend on the zooming dimension, matching those of scalar-reward Lipschitz bandits. This highlights that even without access to absolute rewards, the underlying structure of the problem can still be effectively exploited. As a result, our framework extends the applicability of bandit methods to settings where only relative feedback is available, such as human-in-the-loop optimization and recommender systems.

Several directions remain for future work. First, \ouralgo\  assumes knowledge of the zooming dimension in order to set the exploration depth. This limitation is shared by several zooming-style algorithms, whose guarantees often depend on parameters describing the near-optimal geometry. Removing this requirement, or designing methods that adapt automatically to the unknown zooming dimension, is an important next step. Second, our analysis relies on a bi-Lipschitz transfer condition on the range of utility differences and assumes knowledge of upper Lipschitz constant. Extending the theory to broader preference models, including transfer functions with vanishing slope or weaker identifiability, would further clarify the limits of comparative feedback. Finally, improving computational scalability in high-dimensional action spaces and handling unknown smoothness parameters are natural directions for future work.
\bibliography{references}
\bibliographystyle{icml2026}
\newpage

\appendix

\begin{center}
\Large
    \textbf{Appendix}
\end{center}  

\section{Notation used in the paper }

\begin{table}[h]
\centering
\caption{Notation used throughout the paper}
\label{tab:notation}
\renewcommand{\arraystretch}{1.18}
\setlength{\tabcolsep}{5pt}
\normalsize

\begin{tabular}{l p{0.63\columnwidth}}
\toprule
\textbf{Symbol} & \textbf{Meaning} \\
\midrule

\multicolumn{2}{c}{\textit{Problem Setup}} \\[4pt]

$\mathcal{X}$ & Action space, $\mathcal{X}\subseteq\mathbb{R}^d$ \\
$d$ & Dimension of action space \\
$\|\cdot\|$ & Metric on $\mathcal{X}$ (typically $\ell_\infty$) \\
$f(x)$ & Unknown reward / utility function \\
$\rho(\cdot)$ & Transfer function \\
$\Gamma$ & Lipschitz constant of $\rho$ \\

$P(x \succ y)$ & Probability that arm $x$ is preferred over arm $y$ \\
$\Delta(x,y)$ & Preference gap, $P(x\succ y)-\frac12$ \\
$x^*$ & Condorcet winner / optimal arm \\
$T$ & Time horizon \\

\midrule
\multicolumn{2}{c}{\textit{Algorithmic Quantities}} \\[4pt]

$m$ & Round index \\
$B$ & Maximum depth / total rounds \\
$r_h$ & Region radius (edge length) at level $h$ \\
$\mathcal{C}_h^m$ & Regions at level $h$ in round $m$ \\
$\mathcal{A}_h^m$ & Active regions at level $h$ in round $m$ \\
$d_z$ & Zooming dimension of near-optimal set \\
$X^m$ & Reference arm at round $m$ \\
$\tilde{\Delta}^m$ & Estimated best gap in round $m$ \\
$n_h$ & Samples per region at level $h$ \\

\bottomrule
\end{tabular}
\end{table}

\section{Missing Proofs for Regret Upper Bound}

\subsection{Proof of Lemma \ref{lemma:bound-prob-diff}}
\label{app:lemone}
\LemOne*
\begin{proof}
Fix $(h,m,C)$ and an arbitrary $x\in C$.
We decompose the estimation error by adding and subtracting the expectation:
\begin{align}
\label{eq:error-decomp}
&\big|
\Delta(x,Y)
-
\widehat{\Delta}_h^m(C,Y)
\big|\nonumber\\
&\le
\big|
\widehat{\Delta}_h^m(C,Y)
-
\mathbb{E}\widehat{\Delta}_h^m(C,Y)
\big|
\nonumber\\
&\quad+
\big|
\mathbb{E}\widehat{\Delta}_h^m(C,Y)
-
\Delta(x,Y)
\big|.
\end{align}

Conditional on $Y$ and $C$,
the random variables
$o(x_i, Y)$
are i.i.d.\ Bernoulli and bounded in $[0,1]$.
Therefore, by Hoeffding's inequality,
\begin{align}
\label{eq:hoeffding}
&\mathbb{P}\!\left(
\big|
\widehat{\Delta}_h^m(C,Y)
-
\mathbb{E}\hat{\Delta}_h^m(C,Y)
\big|
\ge
\sqrt{\frac{16\log(T/\delta)}{n_h}}
\right)\\
&\le
\frac{2\delta}{T^8}.\nonumber
\end{align}

By definition of the estimator,
\begin{align*}
\mathbb{E}\widehat{\Delta}_h^m(C,Y)
=
\frac{1}{n_h}
\sum_{i=1}^{n_h} \Delta(x_i,Y).
\end{align*}
Hence,
\begin{align}
\label{eq:bias-ineq1}
&
\big|
\Delta(x,Y)
-
\mathbb{E}\widehat{\Delta}_h^m(C,Y)
\big|
\nonumber\\
&\quad=
\Bigg|
\Delta(x,Y)
-
\frac{1}{n_h}
\sum_{i=1}^{n_h} \Delta(x_i,Y)
\Bigg|.
\end{align}
Applying definition of pairwise preference gap, Jensen's inequality and Assumption \ref{ass:transferlipschitz},

\begin{align}
\label{eq:bias-ineq2}
&\le
\frac{1}{n_h}
\sum_{i=1}^{n_h}
\big|
\rho(f(x)-f(Y))
-
\rho(f(x_i)-f(Y))
\big|
\\
\label{eq:lip}
&\le
\frac{\Gamma}{n_h}
\sum_{i=1}^{n_h}
\big|
f(x)-f(x_i)
\big|.
\end{align}

Since $f$ is $1$-Lipschitz and $\|x-x_i\|\le r_h$ for all $x,x_i\in C$,
\begin{equation}
\label{eq:f-lipschitz}
\big|
f(x)-f(x_i)
\big|
\le
\|x-x_i\|
\le
r_h.
\end{equation}
Substituting into \eqref{eq:bias-ineq2},
\begin{equation}
\label{eq:bias-final}
\big|
\Delta(x,Y)
-
\mathbb{E}\widehat{\Delta}_h^m(C,Y)
\big|
\le
\Gamma r_h
\end{equation}

Combining \eqref{eq:hoeffding} and \eqref{eq:bias-final} with
\eqref{eq:error-decomp},
we obtain that for fixed $(h,m,C,x)$, 
\begin{align*}
\mathbb{P}\Bigg(
&\big|
\Delta(x,X^{m-1})
-
\widehat{\Delta}_h^m(C,X^{m-1})
\big|  \\
&\ge
\sqrt{\frac{16\log(T/\delta)}{n_h}}
+
\Gamma r_h
\Bigg)
\;\le\;
\frac{2\delta}{T^8}.
\end{align*}

Taking a union bound over all
$1\le h\le m\le B$,
all $C\in\mathcal{A}_h^m$,
and the finite set of sampled arms in each cube,
and noting that the total number of such events is polynomial in $T$,
yields
\begin{equation}
\mathbb{P}(\mathcal{E}) \;\ge\; 1-2\delta,
\end{equation}
as claimed.

\end{proof}
\subsection{Proof of Lemma \ref{lem:optimal-not-eliminated}}
\label{app:lemtwo}
\LemTwo*
\begin{proof}
Suppose that the algorithm completes $M\le B$ rounds.
The claim is vacuous for $m=1$, since there is no depth
$h\in\{1,\dots,m-1\}$. Fix an arbitrary completed round
$m\in\{2,\dots,M\}$.
For each depth $h\in\{1,\dots,m-1\}$,
let $(C_h^m)^\star\in\mathcal{A}_h^m$ denote a cube at depth $h$
that contains some Condorcet winner $x^\star$.
We show that $(C_h^m)^\star$ is not eliminated at depth $h$.

Let $C\in\mathcal{A}_{m-1}^{m-1}$ be the cube from which reference arm $X^{m-1}$ was selected at depth $m-1$, i.e., $C_{max}^{m-1}$. Then, under the event $\mathcal{E}$, Lemma~\ref{lemma:bound-prob-diff} implies that
for any sampled arm $x_C\in C$ and for $x^\star\in (C_h^m)^\star$,
\begin{align}
\tilde{\Delta}^{m-1}
&\le
\Delta(x_C, X^{m-1}) + \sqrt{\frac{16\log(T/\delta)}{n_{m-1}}}
+
\Gamma r_{m-1}
\\\label{eq:dev-C}&\le 
\Delta(x_C, X^{m-1}) + \sqrt{\frac{16\log(T/\delta)}{n_{h}}}
+
\Gamma r_{h}
\\
\label{eq:dev-opt}
\widehat{\Delta}_h^m\!\big((C_h^m)^\star\big)
&\ge
\Delta(x^\star, X^{m-1}) - \sqrt{\frac{16\log(T/\delta)}{n_h}} - \Gamma r_h
\end{align}

Subtracting \eqref{eq:dev-opt} from \eqref{eq:dev-C} yields
\begin{align}
\label{eq:gap-diff}
\tilde{\Delta}^{m-1}
-
\widehat{\Delta}_h^m\!\big((C_h^m)^\star\big)
&\le
\Delta(x_C, X^{m-1})
-
\Delta(x^\star, X^{m-1})+
2\sqrt{\frac{16\log(T/\delta)}{n_h}}
+
2\,\Gamma r_h
\end{align}

Since $x^\star$ maximizes $f$ and the preference model
$P(x\succ y)=\rho(f(x)-f(y))$ is monotone in $f(x)$, implying
\begin{equation}
\label{eq:optimal-monotone}
\Delta(x_C, X^{m-1})
-
\Delta(x^\star, X^{m-1})
\le 0 .
\end{equation}
Combining \eqref{eq:gap-diff} and \eqref{eq:optimal-monotone}, we obtain
\begin{equation}
\label{eq:gap-final}
\tilde{\Delta}^{m-1}
-
\widehat{\Delta}_h^m\!\big((C_h^m)^\star\big)
\le
2\sqrt{\frac{16\log(T/\delta)}{n_h}}
+
2\Gamma r_h
\end{equation}

Choosing
\begin{equation}
\label{eq:nh-choice}
n_h
=
\frac{16\log(T/\delta)}{r_h^2},
\end{equation}
we have
\begin{equation}
\label{eq:elim-threshold}
2\sqrt{\frac{16\log(T/\delta)}{n_h}}
+
2\Gamma r_h
=
2r_h +2\Gamma r_h
\le
2\,r_h(1+\Gamma)
\end{equation}

Therefore, the elimination condition is not satisfied for
$(C_h^m)^\star$ at depth $h$, and the cube containing $x^\star$
is not eliminated.

Since the argument holds for all depths
$1\le h\le m-1$ and all completed rounds $m\le M$, the optimal arm
$x^\star$ is never eliminated during the $M$ completed rounds.
\end{proof}

\subsection{Proof of Lemma \ref{lem:shrinking-region-dueling}}
\label{app:lemthree}
\LemThree*
\begin{proof}
Suppose that the algorithm completes $M\le B$ rounds.
Fix an arbitrary completed round $m \in [M]$.
The case $h=1$ is trivial, since
$\Delta(\cdot,\cdot) \in [-\tfrac12,\tfrac12]$
and $r_0 = 1$, so the claim holds.

Now fix $h \ge 2$.
Let $(C_{m-1}^{m-1})^\star \in \mathcal{A}_{m-1}^{m-1}$
denote the unique cube at depth $m-1$ containing
the Condorcet winner $x^\star$.
Under event $\mathcal{E}$, this cube exists and is not eliminated
by Lemma~\ref{lem:optimal-not-eliminated}.

Let $C \in \mathcal{A}_h^m$ and $x \in C$.
Denote by $C_{\mathrm{par}} \in \mathcal{A}_{h-1}^m$
the parent cube of $C$ at depth $h-1$.
Clearly, $x \in C_{\mathrm{par}}$.

Introducing the reference arm $X^{m-1}$ by Lemma \ref{lemma:sti}, if $f(x^\star) \ge f(X^{m-1}) \ge f(x)$ or if $f(x^\star) \ge f(x) \ge f(X^{m-1}) $ , then we have:
\begin{align}
\label{eq:gap-decomp}
\Delta(x^\star,x)
&\leq
\frac{\Gamma}{\lambda}\left(\Delta(x^\star,X^{m-1})
+
\Delta(X^{m-1},x)\right)
\nonumber\\
&\leq
\frac{\Gamma}{\lambda}\left(\Delta(x^\star,X^{m-1})
-
\Delta(x,X^{m-1})\right).
\end{align}

Under event $\mathcal{E}$ and Lemma~\ref{lemma:bound-prob-diff},
we have the following bounds.

\emph{Optimal cube:}
\begin{align}
\label{eq:opt-cube}
\Delta(x^\star,X^{m-1})
&\le \frac{\Gamma}{\lambda}\left(\Delta(x^\star,X^{m-2}) - \Delta(X^{m-1},X^{m-2})\right) \nonumber\\ 
&\le 
\frac{\Gamma}{\lambda}\left(\widehat{\Delta}_{m-1}^{m-1}\!\left((C_{m-1}^{m-1})^\star, X^{m-2}\right) - \widehat{\Delta}_{m-1}^{m-1}(C_{max}^{m-1}, X^{m-2})
+
2\sqrt{\frac{16\log(T/\delta)}{n_{m-1}}}
+
2\Gamma r_{m-1}\right) \nonumber\\
&\le \frac{\Gamma}{\lambda}\left(2\sqrt{\frac{16\log(T/\delta)}{n_{m-1}}}
+
2\Gamma r_{m-1}\right)
\end{align}

\emph{Parent cube:}
\begin{align}
\label{eq:parent-cube}
\Delta(x,X^{m-1})
&\ge
\widehat{\Delta}_{h-1}^m(C_{\mathrm{par}},X^{m-1})
-
\sqrt{\frac{16\log(T/\delta)}{n_{h-1}}}
-
\Gamma r_{h-1}\nonumber\\
&\ge \tilde{\Delta}^{m-1} - 2r_{h-1}(1+\Gamma)-
\sqrt{\frac{16\log(T/\delta)}{n_{h-1}}}
-
\Gamma r_{h-1}
\end{align}
Now,
\begin{align}
\label{eq:deltatilde}
\tilde{\Delta}^{m-1} \ge \Delta(X^{m-1}, X^{m-1}) - \sqrt{\frac{16\log(T/\delta)}{n_{m-1}}}
-
\Gamma r_{m-1}
\end{align}

By the choice
\(
n_h = \frac{16\log(T/\delta)}{r_h^2}
\),
we have
\(
\sqrt{16\log(T/\delta)/n_h} = r_h
\).
Moreover, since $h \le m$, the radius sequence is nonincreasing as $r_h = 2^{-h+1}$,
so
\(
r_{m-1} \le r_{h-1}
\). Along with these facts and substituting \eqref{eq:opt-cube}, \eqref{eq:parent-cube}, and \eqref{eq:deltatilde}
into \eqref{eq:gap-decomp} yields
\begin{align}
\label{eq:combined}
\Delta(x^\star,x)\le \frac{\Gamma}{\lambda}r_{h-1}\left(\frac{2\Gamma}{\lambda}(1+\Gamma) + 4(1+\Gamma)\right) = \kappa(1+\Gamma)(2\kappa + 4)r_{h-1}.
\end{align}
Here $\kappa = \frac{\Gamma}{\lambda}$

Since the argument holds for every depth $1\le h\le m$ and every
completed round $m\in[M]$, the stated shrinking-region property holds
throughout all $M$ completed rounds. If the algorithm stops during
round $M+1$, the partial updates from that incomplete round are discarded,
and hence the property obtained after round $M$ remains unchanged.

\end{proof}

\subsection{Proof of Theorem \ref{thm:regret-theorem}}
\label{app:thmone}
\ThmOne*

\begin{proof}

Fix a positive integer $B$, to be specified later, and let $M\le B$
denote the number of rounds completed by the algorithm. Let $R_m$
denote the regret incurred during a completed round $m$. If $M<B$,
let $R_{M+1}^{\mathrm{part}}$ denote the regret incurred during the
incomplete round $M+1$. Finally, let $R_{\mathrm{rem}}$ denote the
regret incurred after exploration stops and the algorithm repeatedly
plays the last completed reference arm. Thus,
\begin{equation}
\label{eq:regret-decomp}
R(T)
=
\sum_{m=1}^{M}R_m
+
\mathbbm{1}\{M<B\}R_{M+1}^{\mathrm{part}}
+
R_{\mathrm{rem}}.
\end{equation}

We first bound the regret incurred during the completed rounds. Fix a
completed round $m\le M$. During round $m$, the algorithm considers
depths $h=1,\ldots,m$. For every cube
$C\in\mathcal A_h^m$, it performs $n_h$ duels between an arm
$x_{C,i}^m\in C$ and the reference arm $X^{m-1}$. In addition,
$n_m$ duels are performed to compute the new reference arm. Therefore,

\begin{align}
\label{eq:Rm-def}
R_m
&=
\underbrace{
\sum_{h=1}^{m}
\sum_{C\in\mathcal A_h^m}
\sum_{i=1}^{n_h}
\Big(
\Delta(x^\star,x_{C,i}^m)
+
\Delta(x^\star,X^{m-1})
\Big)
}_{\text{Regret from all cubes in one round}}
+
\underbrace{
\sum_{j=1}^{n_m}
\Big(
\Delta(x^\star,z_j^m)
+
\Delta(x^\star,X^m)
\Big)
}_{\text{Regret while selecting $\tilde{\Delta}^m$}}.
\end{align}

where $z^m_j \in C_{\max}^m$ is the arm sampled in the $j^{th}$ additional duel. For convenience, define
\[
K
:=
\kappa(2\kappa+4)(1+\Gamma)
\qquad\text{and}\qquad
L_T:=\log(T/\delta).
\]
By Lemma~\ref{lem:shrinking-region-dueling}, under the event
$\mathcal E$, any arm belonging to a surviving cube at depth $h$
satisfies
\[
\Delta(x^\star,x_{C,i}^m)
\le
Kr_{h-1}.
\]

The reference arm $X^{m-1}$ was selected at the end of round $m-1$ at height $m-1$ , under lemma \ref{lem:shrinking-region-dueling} satisfies

\[
\Delta(x^\star,X^{m-1})
\le
Kr_{m-2}.
\]
Since \(h\le m\), we have
\[
r_{m-2}\le 2r_{h-1},
\]
and therefore
\[
\Delta(x^\star,X^{m-1})
\le
2Kr_{h-1}.
\]

Consequently, each duel performed at depth $h$ incurs regret at most
\begin{align}
\Delta(x^\star,x_{C,i}^m)
+
\Delta(x^\star,X^{m-1})
&\le
3Kr_{h-1}
\nonumber\\
&=
6Kr_h,
\label{eq:per-duel}
\end{align}
where we used $r_{h-1}=2r_h$.

Since
$
n_h
=
\frac{16L_T}{r_h^2},
$
the regret contributed by a single cube at depth $h$ is at most
$
n_h\cdot 6Kr_h
=
\frac{96KL_T}{r_h}.
$
Therefore,
\begin{equation}
\label{eq:Rm-preliminary}
R_m
\le
\sum_{h=1}^{m}
|\mathcal A_h^m|
\frac{96KL_T}{r_h}
+
A_m,
\end{equation}
where $A_m$ denotes the regret incurred by the additional $n_m$
duels used to compute $\tilde{\Delta}^m$.

We now calculate $A_m$. By
Lemma~\ref{lem:shrinking-region-dueling},
\[
\Delta(x^\star,X^{m})
\le
Kr_{m-1} ,\quad \Delta(x^\star,z_j^m)
\le
Kr_{m-1} 
\]
Hence,
\begin{align}
A_m
&:=
\sum_{j=1}^{n_m}
\Big(
\Delta(x^\star,z_j^m)
+
\Delta(x^\star,X^{m})
\Big)
\nonumber\\
&\le
2n_mKr_{m-1}
\nonumber\\
&=
\frac{32L_T}{r_m^2}Kr_{m-1}
\nonumber\\
&=
\frac{64KL_T}{r_m},
\label{eq:Am-def}
\end{align}
where the last equality follows from $r_{m-1}=2r_m$.

We next bound the number of surviving cubes. By
Lemma~\ref{lem:shrinking-region-dueling}, every cube in
$\mathcal A_h^m$ is contained in
\[
\left\{
x\in\mathcal X:
\Delta(x^\star,x)
\le
Kr_{h-1}
\right\}
=
\mathcal S(2Kr_h).
\]
Each cube in $\mathcal A_h^m$ contains an $\ell_\infty$-ball of radius
$r_h/2$, and these balls have disjoint interiors. Thus,
$\mathcal A_h^m$ induces an $(r_h/2)$-packing of
$\mathcal S(2Kr_h)$.

By the definition of the zooming dimension, the set
$\mathcal S(2Kr_h)$ can be covered by at most
\[
C_z(2Kr_h)^{-d_z}
\]
balls of radius $2Kr_h$. Moreover, each such ball intersects at most
$
(8K+2)^d
$
disjoint $\ell_\infty$-balls of radius $r_h/2$. Therefore, defining
\[
\widetilde C_z
:=
C_z(2K)^{-d_z}(8K+2)^d,
\]
we obtain
\begin{equation}
\label{eq:packing}
|\mathcal A_h^m|
\le
\widetilde C_z r_h^{-d_z}
=
\widetilde C_z2^{(h-1)d_z}.
\end{equation}
By increasing $\widetilde C_z$ if necessary, we may assume that
$\widetilde C_z\ge1$.

Substituting \eqref{eq:packing} into
\eqref{eq:Rm-preliminary} gives
\begin{align}
R_m
&\le
96K\widetilde C_zL_T
\sum_{h=1}^{m}
2^{(h-1)(d_z+1)}
+
A_m.
\label{eq:Rm-final}
\end{align}

We now sum over all completed rounds $m\le M$. Exchanging the order of
summation yields
\begin{align}
\sum_{m=1}^{M}R_m
&\le
96K\widetilde C_zL_T
\sum_{m=1}^{M}
\sum_{h=1}^{m}
2^{(h-1)(d_z+1)}
+
\sum_{m=1}^{M}A_m
\nonumber\\
&=
96K\widetilde C_zL_T
\sum_{h=1}^{M}
(M-h+1)2^{(h-1)(d_z+1)}
+
\sum_{m=1}^{M}A_m.
\label{eq:completed-sum}
\end{align}

Let
\[
q:=2^{d_z+1}.
\]
Since $d_z\ge0$, we have $q\ge2$. Thus,
\begin{align}
\sum_{h=1}^{M}(M-h+1)q^{h-1}
&=
q^{M-1}
\sum_{j=0}^{M-1}(j+1)q^{-j}
\nonumber\\
&\le
q^{M-1}
\sum_{j=0}^{\infty}(j+1)q^{-j}
\nonumber\\
&=
\frac{q^{M-1}}{(1-q^{-1})^2}
\nonumber\\
&\le
4q^{M-1}.
\label{eq:double-sum}
\end{align}

Moreover, using $r_m=2^{-m+1}$ and \eqref{eq:Am-def},
\begin{align}
\sum_{m=1}^{M}A_m
&\le
64KL_T
\sum_{m=1}^{M}\frac{1}{r_m}
\nonumber\\
&=
64KL_T
\sum_{m=1}^{M}2^{m-1}
\nonumber\\
&=
64KL_T(2^M-1)
\nonumber\\
&\le
128K\widetilde C_zL_Tq^{M-1}.
\label{eq:Am-sum}
\end{align}
Here we used $\widetilde C_z\ge1$ and
$2^{M-1}\le q^{M-1}$.

Combining \eqref{eq:completed-sum}, \eqref{eq:double-sum}, and
\eqref{eq:Am-sum}, we obtain
\begin{equation}
\label{eq:completed-regret}
\sum_{m=1}^{M}R_m
\le
C_0K\widetilde C_zL_T
2^{(M-1)(d_z+1)}
\end{equation}
for some universal constant $C_0>0$.

We now consider separately the cases $M=B$ and $M<B$.

\paragraph{Case 1: $M=B$.}
In this case, all $B$ rounds are completed. By
Lemma~\ref{lem:shrinking-region-dueling}, the final reference arm
satisfies
\[
\Delta(x^\star,X^B)
\le
Kr_{B-1}.
\]
Therefore, repeatedly playing $X^B$ during the remaining time steps
incurs regret at most
\begin{align}
R_{\mathrm{rem}}
&\le
2Kr_{B-1}T
\nonumber\\
&=
4K2^{-B+1}T.
\label{eq:late-regret}
\end{align}

Combining \eqref{eq:completed-regret} and
\eqref{eq:late-regret}, we obtain
\begin{equation}
\label{eq:combined-regret}
\begin{aligned}
R(T)
\le\;&
C_0K\widetilde C_zL_T
2^{(B-1)(d_z+1)}
\\
&+
4K2^{-B+1}T.
\end{aligned}
\end{equation}

\paragraph{Case 2: $M<B$.}
In this case, the algorithm stops during round $M+1$ because
insufficient time remains to complete the next sampling block.
Although the partial updates from round $M+1$ are discarded, the
regret incurred by the duels already performed in this round must
still be counted.

Every sampling block that is started during round $M+1$ is completed
before the algorithm proceeds to the next block. Furthermore, the
incomplete round visits only a subset of the cubes that could be
visited in a complete round $M+1$. The argument of
Lemma~\ref{lem:shrinking-region-dueling} therefore applies to all cubes
visited before termination. Consequently, the same argument used in
\eqref{eq:Rm-final} gives
\begin{align}
R_{M+1}^{\mathrm{part}}
&\le
96K\widetilde C_zL_T
\sum_{h=1}^{M+1}
2^{(h-1)(d_z+1)}
+
A_{M+1}.
\label{eq:partial-round-first}
\end{align}

Since $q=2^{d_z+1}\ge2$,
\[
\sum_{h=1}^{M+1}q^{h-1}
=
\frac{q^{M+1}-1}{q-1}
\le
2q^M.
\]
Also,
\begin{align}
A_{M+1}
&\le
\frac{64KL_T}{r_{M+1}}
\nonumber\\
&=
64KL_T2^M
\nonumber\\
&\le
64K\widetilde C_zL_Tq^M.
\end{align}
Thus,
\begin{equation}
\label{eq:partial-round-regret}
R_{M+1}^{\mathrm{part}}
\le
C_1K\widetilde C_zL_T
2^{M(d_z+1)}
\end{equation}
for some universal constant $C_1>0$.

Combining \eqref{eq:completed-regret} and
\eqref{eq:partial-round-regret}, we obtain
\begin{equation}
\label{eq:exploration-early-stop}
\sum_{m=1}^{M}R_m
+
R_{M+1}^{\mathrm{part}}
\le
C_2K\widetilde C_zL_T
2^{M(d_z+1)}
\end{equation}
for some universal constant $C_2>0$.

Let $\tau$ denote the number of time steps remaining when the
algorithm stops. Suppose that the algorithm stops before starting a
sampling block of size $n_h$ during round $M+1$. By the stopping
condition,
\[
\tau<n_h.
\]
Since $h\le M+1$ and $n_h$ is nondecreasing in $h$,
\begin{equation}
\label{eq:remaining-rounds}
\tau<n_h\le n_{M+1}.
\end{equation}

By Lemma~\ref{lem:shrinking-region-dueling}, the reference arm $X^M$
obtained from the last completed round satisfies
\begin{equation}
\label{eq:last-reference-gap}
\Delta(x^\star,X^M)
\le
Kr_{M-1}.
\end{equation}
Therefore, repeatedly playing $X^M$ during the remaining $\tau$ steps
incurs regret at most
\begin{align}
R_{\mathrm{rem}}
&\le
2Kr_{M-1}\tau
\nonumber\\
&<
2Kr_{M-1}n_{M+1}.
\label{eq:early-cleanup}
\end{align}

Using
\[
n_{M+1}
=
\frac{16L_T}{r_{M+1}^2},
\qquad
r_{M-1}=4r_{M+1},
\qquad
r_{M+1}=2^{-M},
\]
we obtain
\begin{align}
R_{\mathrm{rem}}
&<
32KL_T\frac{r_{M-1}}{r_{M+1}^2}
\nonumber\\
&=
128KL_T\frac{1}{r_{M+1}}
\nonumber\\
&=
128KL_T2^M.
\label{eq:early-cleanup-bound}
\end{align}
Since $d_z\ge0$,
\[
2^M
\le
2^{M(d_z+1)}.
\]
Combining \eqref{eq:exploration-early-stop} and
\eqref{eq:early-cleanup-bound} gives
\begin{equation}
\label{eq:early-stop-total}
R(T)
\le
C_3K\widetilde C_zL_T
2^{M(d_z+1)}
\end{equation}
for some universal constant $C_3>0$.

Finally, choose
\begin{equation}
\label{eq:B-choice}
B^\star
=
\left\lceil
\frac{\log_2T}{d_z+2}
\right\rceil+1.
\end{equation}
When $M=B$, this choice balances the two terms in
\eqref{eq:combined-regret}, giving
\[
R(T)
=
\widetilde O\!\left(
T^{\frac{d_z+1}{d_z+2}}
\right).
\]
When $M<B$, we have $M\le B-1$, and hence
\begin{align}
2^{M(d_z+1)}
&\le
2^{(B-1)(d_z+1)}
\nonumber\\
&=
O\!\left(
T^{\frac{d_z+1}{d_z+2}}
\right).
\end{align}
Thus, the same regret bound holds even if the algorithm stops before
completing all $B$ rounds:
\[
R(T)
=
\widetilde O\!\left(
T^{\frac{d_z+1}{d_z+2}}
\right).
\]
This completes the proof.
\end{proof}

\subsection{Proof of Lemma \ref{lemma:Space-analysis}}
\label{lemfive}
\label{app:space}
\LemmaFour*
\begin{proof}
We provide a complete accounting of all quantities stored by the algorithm
at any time step.
\paragraph{Quantities maintained.}
At any timestamp, the algorithm stores only a constant number of variables:
\begin{itemize}
    \item Previous-round reference arm and gap: $(X^{m-1}, \tilde{\Delta}^{m-1})$.
    \item Current-round best arm and gap: $(X^{m}, \tilde{\Delta}^{m})$.
    \item Current best cube and gap: $(C_{\max}^m, \Delta_{\max}^m)$.
    \item Current cube and its estimate: $(C, \hat{\Delta}_h^m(C))$.
    \item Counters and indices: $(t, m, h)$ and the loop index for iterating over subcubes.
\end{itemize}
Thus, the number of stored objects is constant.

\paragraph{Encoding of cubes.}
Each cube $C$ lies in a hierarchical partition of $\mathcal{X} = [0,1]^d$.
A cube at depth $h$ is identified by its grid index in each dimension,
which requires $O(dh)$ bits. Since the maximum depth is $B = O(\log T)$,
each cube can be encoded using $O(d \log T)$ bits.

Alternatively, since each explored cube must be sampled at least once,
the total number of explored cubes is at most $T$. Hence, we can assign
each explored cube a unique identifier using $\log T$ bits.
Both representations yield the same asymptotic bound.

\paragraph{Encoding of arms.}
Each arm $X^m$ is sampled from some cube. We represent $X^m$ using the
identifier of the cube from which it was sampled together with a finite-precision
offset (or equivalently, a random seed identifying the sample).
Since the total number of samples is at most $T$, such a representation
requires $O(d \log T)$ bits.

\paragraph{Encoding of empirical gaps.}
For dueling feedback, the empirical gap is computed from binary outcomes
$o_i \in \{0,1\}$ as
\[
\hat{\Delta} = \frac{1}{n} \sum_{i=1}^n o_i - \frac{1}{2}.
\]
Thus, it suffices to store:
\begin{itemize}
    \item the number of comparisons $n \le T$,
    \item the number of wins $\sum_{i=1}^n o_i \le T$.
\end{itemize}
Each of these requires at most $\log T$ bits, so each empirical gap
requires $2\log T$ bits.

\paragraph{Counters.}
The counters $t, m, h$, as well as loop indices used to enumerate subcubes,
are all bounded by $T$ and hence require at most $\log T$ bits each.

\paragraph{Bounding the number of stored quantities.}
Since only a constant number of arms, cubes, empirical gaps, and counters
are stored simultaneously, and each requires at most $O(d \log T)$ bits,
the total working memory is $O(d \log T)$ bits.

\paragraph{Recursion and memory reuse.}
Although Algorithm~\ref{alg:roundfunc} is described recursively, it can be implemented
in a depth-first, in-place manner. Specifically, subcubes of a cube are
generated sequentially (e.g., in lexicographic order) and processed one
at a time. After a subcube is processed, all variables associated with it
are overwritten.

Therefore, the algorithm never stores:
\begin{itemize}
    \item the full set of subcubes,
    \item the full recursion tree,
    \item or the full history of comparisons.
\end{itemize}

Although algorithm is written recursively, it is implemented as an in-place
depth-first traversal. The algorithm stores only the current cube, the current
depth, and a constant number of counters and empirical estimates. Subcubes are
generated and processed sequentially, and their local variables are overwritten
after use. Hence, no recursion stack is stored, and no additional memory
proportional to the recursion depth is required.
Hence, no additional space proportional to the number of subcubes or
recursion depth is required.

\paragraph{Conclusion.}
Combining the above bounds, the total memory required is $O(d \log T)$ bits.
When the dimension $d$ is treated as a constant, this simplifies to
$O(\log T)$ bits.
Algorithm~\ref{alg:roundfunc}\end{proof}

\section{Proof for Regret Lower Bound}
\subsection{Proof of Theorem \ref{thm:alg-reduction}}
\label{app:lowerbound}
\ThmTwo*
\begin{proof}
The algorithm $B_A$ internally runs $A$. Whenever $A$ queries a pair $(x_t,y_t)$,
the algorithm $B_A$ pulls $x_t$ and $y_t$ once each, obtaining independent
samples
\[
U_t\sim\mathrm{Ber}(\mu(x_t)),
\qquad
V_t\sim\mathrm{Ber}(\mu(y_t)).
\]
It then generates a synthetic duel outcome
\[
O_t=
\begin{cases}
1,& U_t>V_t,\\
0,& U_t<V_t,\\
C_t,& U_t=V_t,
\end{cases}
\]
where $C_t\sim\mathrm{Ber}(1/2)$ is independent, and feeds $O_t$ back to the
internal copy of $A$.

A direct calculation shows
\[
\mathbb P(O_t=1\mid x_t,y_t)
=
\frac12\bigl(1+\mu(x_t)-\mu(y_t)\bigr),
\]
so the synthetic duel has exactly the same distribution as a true comparison in
the induced linear-link dueling instance. Hence the internal transcript of $A$
inside $B_A$ has the correct law.

Let $x^\star\in\arg\max_x \mu(x)$. Under the linear link,
\[
\Delta(x^\star,x)=\frac12\bigl(\mu(x^\star)-\mu(x)\bigr).
\]
Therefore, at round $t$,
\[
\Delta(x^\star,x_t)+\Delta(x^\star,y_t)
=
\frac12\bigl(\mu(x^\star)-\mu(x_t)\bigr)
+\frac12\bigl(\mu(x^\star)-\mu(y_t)\bigr).
\]
The two scalar pulls used by $B_A$ at that step incur regret
\[
\bigl(\mu(x^\star)-\mu(x_t)\bigr)+\bigl(\mu(x^\star)-\mu(y_t)\bigr),
\]
which is exactly twice the dueling regret. Summing over $t=1,\dots,T$ gives
\[
R^{\mathrm{sc}}_{B_A}(2T)=2\,R^{\mathrm{duel}}_{A}(T)
\]
\end{proof}

\section{Bi-Lipschitz continuity of transfer functions}
\begin{lemma}[Assumption \ref{ass:transferlipschitz} logistic sigmoid]
\label{lem:sigmoid-lipschitz}

Let $\rho:[-1,1]\to(0,1)$ be the logistic sigmoid defined by
\begin{equation}
\label{eq:sigmoid-def}
\rho(z) := \frac{1}{1+e^{-z}} .
\end{equation}
Then $\rho(0) = \frac12$, $\rho(z) = 1-\rho(-z)$, and $\rho$ is bi-Lipschitz with $\Gamma = \tfrac14$ and $\lambda = \frac{e}{(1+e)^2}$.
\end{lemma}

\begin{proof}
$\rho(0) = \frac12$, $\rho(z) = 1-\rho(-z)$ are easy to verify. For bi-Lipschitz,
we first compute the derivative of $\rho$. 
Differentiating yields
\begin{equation}
\label{eq:sigmoid-derivative}
\rho'(z)
=
\frac{e^{-z}}{(1+e^{-z})^2}
=
\frac{e^z}{(1+e^z)^2}
=
\rho(z)\bigl(1-\rho(z)\bigr).
\end{equation}

Since \(\rho(z)\in(0,1)\) for all \(z\in\mathbb R\), the function
\(u\mapsto u(1-u)\) is maximized over \(u\in[0,1]\) at \(u=1/2\).
Therefore,
\begin{equation}
\label{eq:sigmoid-upper-derivative-bound}
\sup_{z\in\mathbb R} |\rho'(z)|
=
\rho(0)(1-\rho(0))
=
\frac14 .
\end{equation}
Applying mean value theorem, we get that the logistic transfer is globally Lipschitz with upper constant
\[
\Gamma=\frac14.
\]

The derivative \(\rho'(z)\) is symmetric around zero and decreases as
\(|z|\) increases. Hence, over \([-1,1]\), its minimum is attained at
\(z=\pm 1\). Thus,
\begin{equation}
\label{eq:sigmoid-lower-derivative-bound}
\inf_{z\in[-1,1]} \rho'(z)
=
\rho'(1)
=
\rho'(-1)
=
\frac{e}{(1+e)^2}.
\end{equation}

Thus, we obtain
\[
\frac{e}{(1+e)^2}(a-b)
\le
\rho(a)-\rho(b)
\le
\frac14(a-b).
\]

Therefore, for all \(-1\le b\le a\le 1\), the logistic transfer satisfies
\[
\lambda(a-b)
\le
\rho(a)-\rho(b)
\le
\Gamma(a-b),
\]
with
\[
\lambda=\frac{e}{(1+e)^2},
\qquad
\Gamma=\frac14.
\]
\end{proof}

\begin{lemma}[Assumption \ref{ass:transferlipschitz} for the probit transfer]
\label{lem:probit-lipschitz}

Let $\Phi:\mathbb{R}\to(0,1)$ denote the probit function, defined as the
cumulative distribution function of a standard normal random variable,
\begin{equation}
\label{eq:probit-def}
\Phi(z)
:=
\int_{-\infty}^{z}
\frac{1}{\sqrt{2\pi}} e^{-t^{2}/2}\,dt .
\end{equation}
Then $\Phi(0) = \frac12$, $\Phi(z) = 1-\Phi(-z)$, and 
$\Phi$ is bi-Lipschitz continuous with 
$\Gamma = \tfrac{1}{\sqrt{2\pi}}$ and $\lambda = \tfrac{1}{\sqrt{2\pi}}e^{-1/2}$.
\end{lemma}

\begin{proof}
The function naturally satisfies $\Phi(0) = \frac12$ and symmetry $\Phi(z) = 1-\Phi(-z)$.
Since $\Phi$ is differentiable on $\mathbb{R}$, its derivative is given by
the standard normal density:
\begin{equation}
\label{eq:probit-derivative}
\Phi'(z)
=
\frac{1}{\sqrt{2\pi}} e^{-z^{2}/2},
\qquad z\in[-1,1].
\end{equation}
The exponential term achieves its maximum at $z=0$, from which it follows that
\begin{equation}
\label{eq:probit-derivative-bound}
|\Phi'(z)|
\le
\frac{1}{\sqrt{2\pi}},
\qquad
\forall z\in[-1,1].
\end{equation}

Consider now arbitrary $a,b\in\mathbb{R}$. By the Mean Value Theorem, there exists
a point $\xi$ between $a$ and $b$ such that
\begin{equation}
\label{eq:probit-mvt}
\Phi(a)-\Phi(b)
=
\Phi'(\xi)\,(a-b).
\end{equation}
Combining this identity with the uniform bound
\eqref{eq:probit-derivative-bound} yields
\[
|\Phi(a)-\Phi(b)|
\le
|\Phi'(\xi)|\,|a-b|
\le
\tfrac{1}{\sqrt{2\pi}}\,|a-b|,
\]
We now prove the lower bi-Lipschitz bound on the interval \([-1,1]\). Since \(e^{-z^2/2}\) decreases as \(|z|\) increases, the minimum of\(\Phi'(z)\) over \([-1,1]\) is attained at \(z=\pm 1\). Hence
\begin{equation}
\label{eq:probit-lower-derivative-bound}
\inf_{z\in[-1,1]} \Phi'(z)=\Phi'(1)=\Phi'(-1)=\frac{1}{\sqrt{2\pi}}e^{-1/2}.
\end{equation}
Now fix arbitrary \(-1\le b\le a\le 1\). By the Mean Value Theorem,there exists \(\xi\in[b,a]\subseteq[-1,1]\) such that
\begin{equation}
\label{eq:probit-mvt}
\Phi(a)-\Phi(b)=\Phi'(\xi)(a-b) \ge \frac{1}{\sqrt{2\pi}}e^{-1/2}.
\end{equation}
Thus, the probit transfer is \((\lambda,\Gamma)\)-bi-Lipschitz on\([-1,1]\).
\end{proof}

\begin{lemma}[Assumption \ref{ass:transferlipschitz} for linear link]
\label{lem:linear-lipschitz}

Let $\rho:\mathbb{R}\to\mathbb{R}$ be defined by
\begin{equation}
\label{eq:linear-def}
\rho(z)
:=
\tfrac12(1+z).
\end{equation}
Then $\rho(0)=\frac12$, $\rho(z) = 1-\rho(-z)$, and $\rho$ is bi-Lipschitz on $[-1,1]$ with $\Gamma = \tfrac12$ and $\lambda = \frac12$  i.e.,
\end{lemma}

\begin{proof}
It is easy to see $\rho(0)=\frac12$ and $\rho(z) = 1-\rho(-z)$. For bi-Lipschitz, fix arbitrary $a,b\in\mathbb{R}$. By direct computation,
\begin{equation}
\label{eq:linear-diff}
\rho(a)-\rho(b)
=
\tfrac12(1+a)-\tfrac12(1+b)
=
\tfrac12(a-b).
\end{equation}
This gives the value of $\Gamma = \lambda = \frac12$.
\end{proof}

\section{Proof of Lemma \ref{lemma:sti}}
\label{lemma:sti-transfer}
\LemmaZero*
\begin{proof}
When $f(x) \ge f(y) \ge f(z)$, we have:
\begin{align*}
\Delta(x,y) &= \rho(f(x)-f(y)) -\rho(0) \ge \lambda(f(x) - f(y))\\
\Delta(y,z) &= \rho(f(y)-f(z)) - \rho(0) \ge \lambda(f(y)-f(z))\\
\Delta(x,z) &= \rho(f(x) - f(z)) - \rho(0) \le \Gamma(f(x) - f(y)) + \Gamma(f(y)-f(z))\\
&\le \frac{\Gamma}{\lambda}(\Delta(x,y) + \Delta(y,z)) 
\end{align*}
When $f(x) \ge f(z) \ge f(y)$, we have: 
\begin{align*}
\Delta(x,z) &= \rho(f(x) - f(z)) - \rho(0) \le \Gamma(f(x) - f(z))  \\
&\le \frac{\Gamma}{\lambda}[\rho(f(x)-f(y)) - \rho(f(z) - f(y))]\\
&\le \frac{\Gamma}{\lambda}(\Delta(x,y) + \Delta(y,z)) 
\end{align*}
Here, second last inequality follows from $\rho(f(x)-f(y)) - \rho(f(z) - f(y)) \ge \lambda(f(x)-f(z))$.
\end{proof}

\section{Additional Experimental Results}
\label{sec:Additional_experiment}
\subsection{Experiments with Random Initial Points}

\begin{figure}[h]
    \centering

    \begin{subfigure}[t]{0.32\linewidth}
        \centering
        \includegraphics[width=\linewidth]{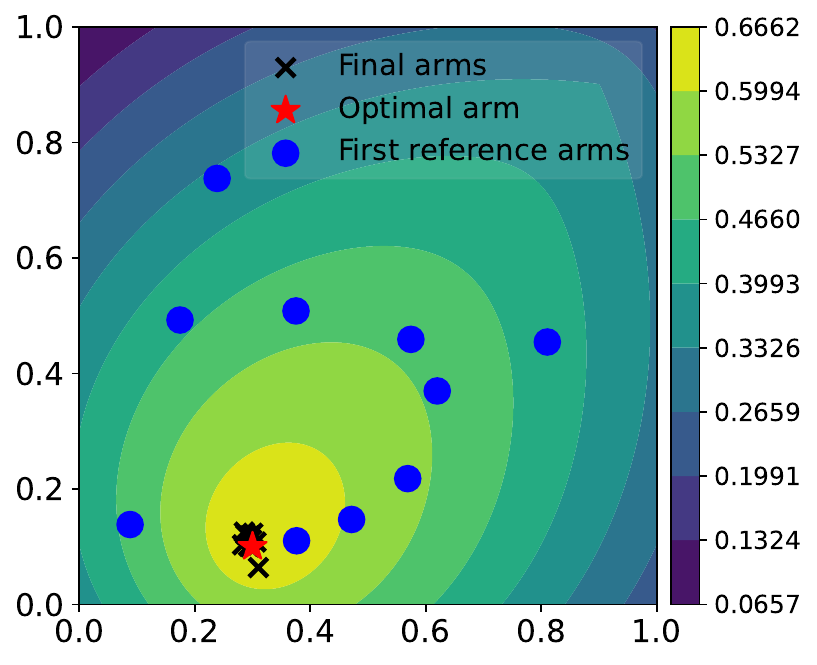}
        \caption{Reference arm convergence}
        \label{fig:rand_ref_conv}
    \end{subfigure}
    \hfill
    \begin{subfigure}[t]{0.32\linewidth}
        \centering
        \includegraphics[width=\linewidth]{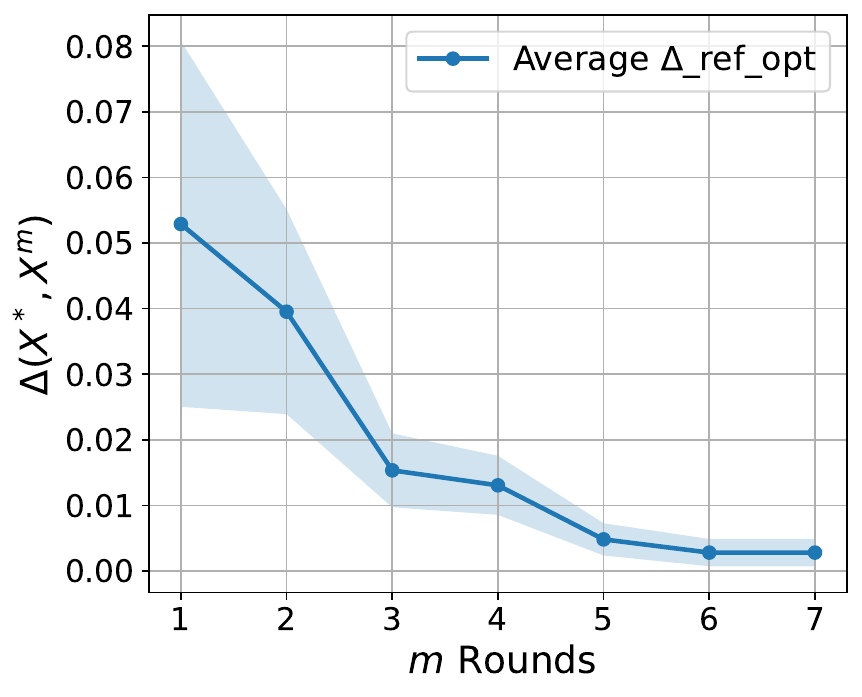}
        \caption{Convergence of $\Delta(x^*,X^m)$}
        \label{fig:rand_delta_conv}
    \end{subfigure}
    \hfill
    \begin{subfigure}[t]{0.32\linewidth}
        \centering
        \includegraphics[width=\linewidth]{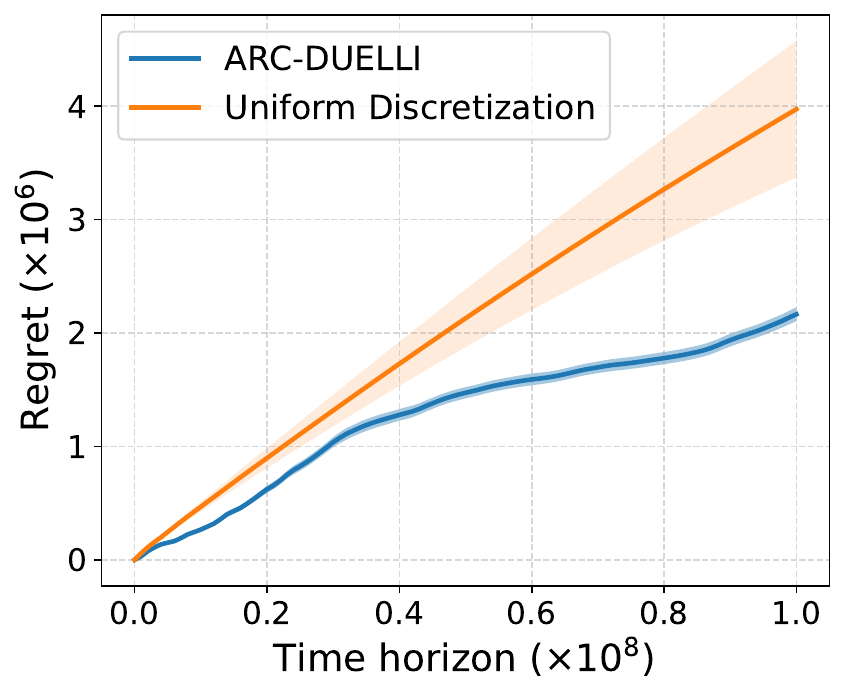}
        \caption{Regret comparison}
        \label{fig:rand_regret}
    \end{subfigure}

    \caption{
    Performance of \ouralgo\ with randomly initialized reference arms over 10 runs. 
    (a) Reference arms converge toward the optimal arm. 
    (b) The optimality gap $\Delta(x^*, X^m)$ decreases across rounds. 
    (c) Cumulative regret remains comparable to the fixed-initialization setting.
    }

    \label{fig:random_init}
\end{figure}
Even when initialized with a randomly chosen reference arm for the same reward function that we use in the main paper, \ouralgo\ maintains strong performance. As shown in Figure~\ref{fig:rand_ref_conv}, the reference arms consistently converge toward the optimal arm across runs, indicating robustness to initialization. Figure~\ref{fig:rand_delta_conv} further shows that the optimality gap $\Delta(x^*, X^m)$ decreases steadily over rounds and approaches zero, demonstrating effective learning. Finally, Figure~\ref{fig:rand_regret} shows that the cumulative regret remains low and continues to outperform the uniform discretization baseline. Overall, these results highlight that \ouralgo\ is insensitive to the choice of initial reference arm and reliably converges to near-optimal solutions.

\subsection{Experiments with Scaled Metric Function}

\begin{figure}[h]
    \centering

    \begin{subfigure}[t]{0.32\linewidth}
        \centering
        \includegraphics[width=\linewidth]{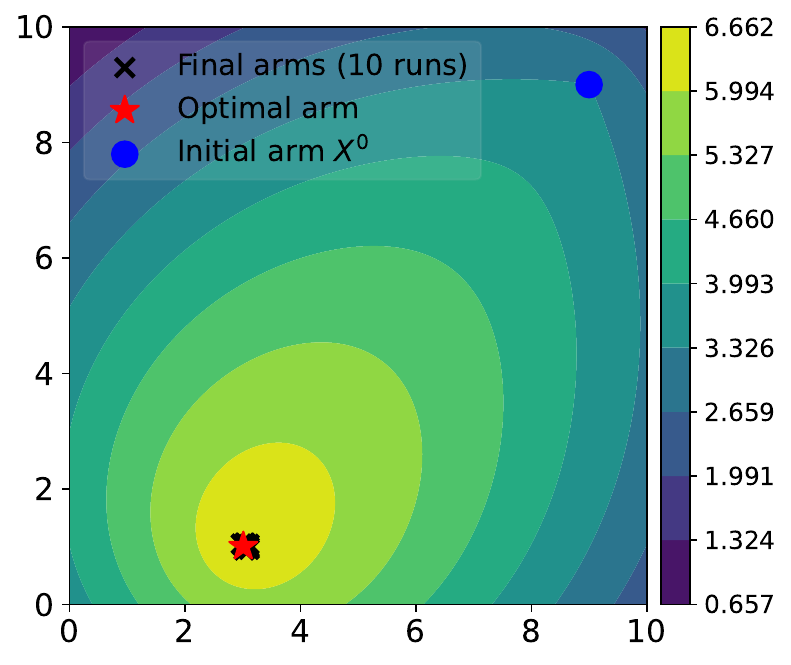}
        \caption{Reference arm convergence}
        \label{fig:scaled_ref_conv}
    \end{subfigure}
    \hfill
    \begin{subfigure}[t]{0.32\linewidth}
        \centering
        \includegraphics[width=\linewidth]{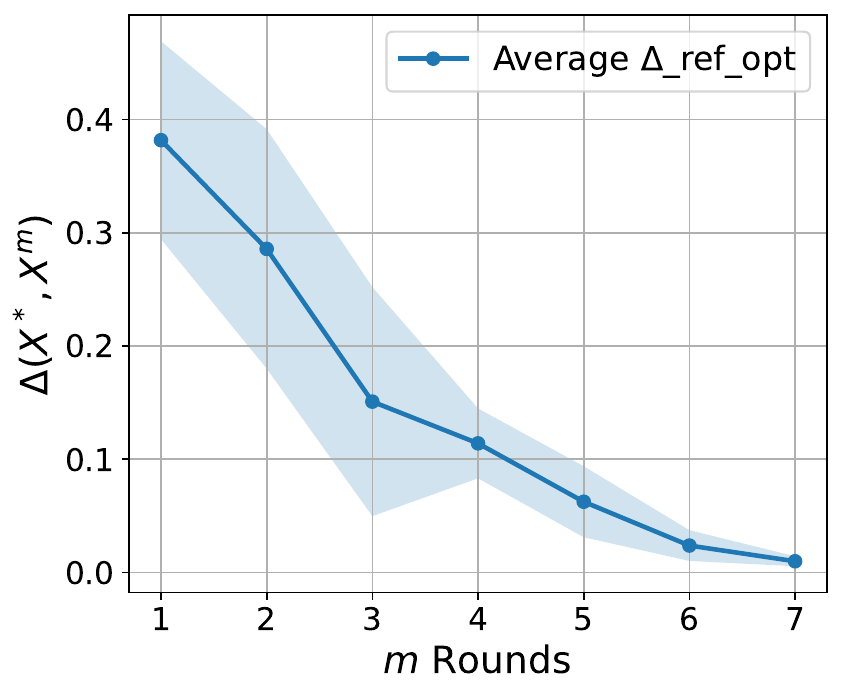}
        \caption{Convergence of $\Delta(x^*,X^m)$}
        \label{fig:scaled_delta_conv}
    \end{subfigure}
    \hfill
    \begin{subfigure}[t]{0.32\linewidth}
        \centering
        \includegraphics[width=\linewidth]{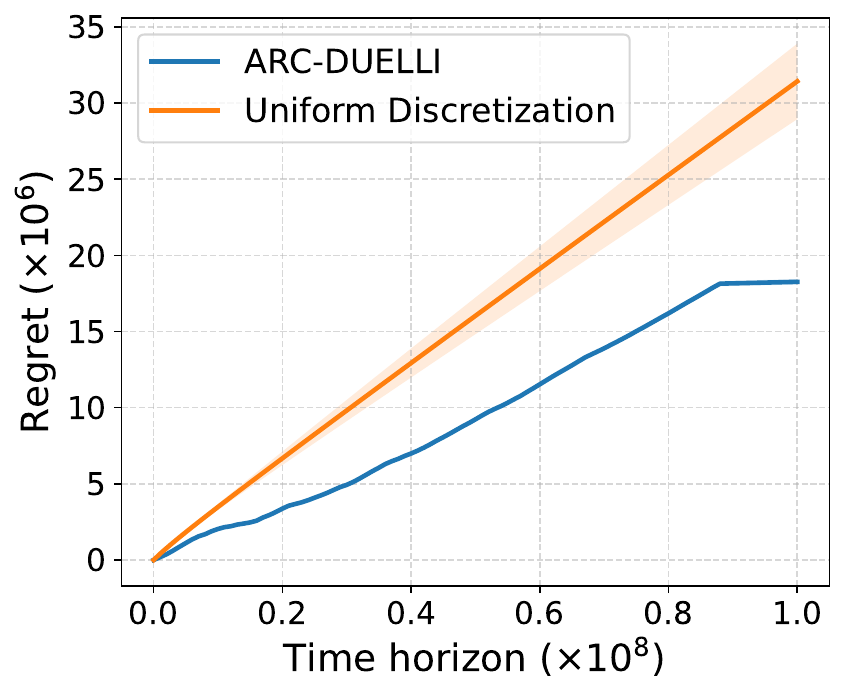}
        \caption{Regret comparison}
        \label{fig:scaled_regret}
    \end{subfigure}

    \caption{Performance of \ouralgo\ under scaled rewards in $[0,10]$. 
(a) Reference arms converge toward the optimal arm despite the larger reward scale. 
(b) The gap $\Delta(x^\star, X^m)$ decreases across rounds, indicating consistent improvement in the reference arm. 
(c) Cumulative regret compared with a uniform discretization baseline, showing superior performance of \ouralgo. The solid curves denote averages over runs, and the shaded regions represent $\pm 1$ standard deviation.}

    \label{fig:scaled_results}
\end{figure}


We further evaluate \ouralgo\ beyond the unit domain by considering a scaled action space $\mathcal{X} = [0,10]^2$. The utility function is defined as$\mu(x) = 10 - \frac{2}{3}\|x - x_1\| - \frac{1}{3}\|x - x_2\|, $ where $x_1 = (3.0, 1.0)$ and $x_2 = (9.0, 9.0)$. This reward landscape induces a unique maximizer at $x^\star = (3.0, 1.0)$.To account for the larger domain, the Lipschitz constant of the reward function changes accordingly. We incorporate this by replacing the unit Lipschitz constant with a general constant $L$ in the analysis \ref{eq:lip}, leading to a slight modification of the elimination condition from $2r_h(1+\Gamma)$ to $2r_h(1 + L\Gamma)$. The baseline discretization is constructed using the same resolution rule as in the unit domain. These experiments demonstrate that \ouralgo\ is not restricted to the $[0,1]^d$ setting. As shown in Figure~\ref{fig:scaled_ref_conv}, the final reference arms continue to concentrate near the optimal arm, indicating that the exploration and elimination strategy remains effective under scaling. Figure~\ref{fig:scaled_delta_conv} shows that the optimality gap $\Delta(x^\star, X^m)$ steadily decreases across rounds and approaches zero, confirming consistent convergence behavior. Finally, Figure~\ref{fig:scaled_regret} illustrates that \ouralgo\ maintains significantly lower cumulative regret compared to uniform discretization, even in the scaled setting. 

Overall, these results highlight that the performance of \ouralgo\ extends naturally to more general domains, with robustness to changes in both the reward scale and the underlying action space.

\subsection{Experiments with different reward function}

To demonstrate that \ouralgo\ is robust to a wide range of utility landscapes, we evaluate it on six benchmark reward functions that differ significantly in the geometry around their optimum. These functions range from sharp unimodal objectives to flat plateaus, ring-shaped optima, and multimodal landscapes, thereby testing the algorithm under varying levels of optimization difficulty.

\paragraph{Flat ($d_z= \frac{4}{3}$).}
The Flat reward function is designed to be particularly challenging by introducing a large nearly-optimal region around the maximizer. Near the optimum at $(0.78,0.22)$, the utility decreases cubically, producing very small reward differences between nearby actions. Consequently, distinguishing the optimal arm requires substantially more comparisons than in functions with steeper gradients.

Let $d(x)=\|x-(0.78,0.22)\|_2$ and $r=0.22$. The utility is defined as
\[
\mu(x)=\operatorname{clip}(1-g(x),0,1),
\]
where
\[
g(x)=
\begin{cases}
\dfrac{0.9d(x)^3}{3r^2}, & d(x)\le r,\\[1mm]
0.9\left(d(x)-\dfrac{2r}{3}\right), & \text{otherwise}.
\end{cases}
\]

The experimental results are shown in Fig.~\ref{fig:flat_results}.

\begin{figure}[h]
    \centering

    \begin{subfigure}[t]{0.32\linewidth}
        \centering
        \includegraphics[width=\linewidth]{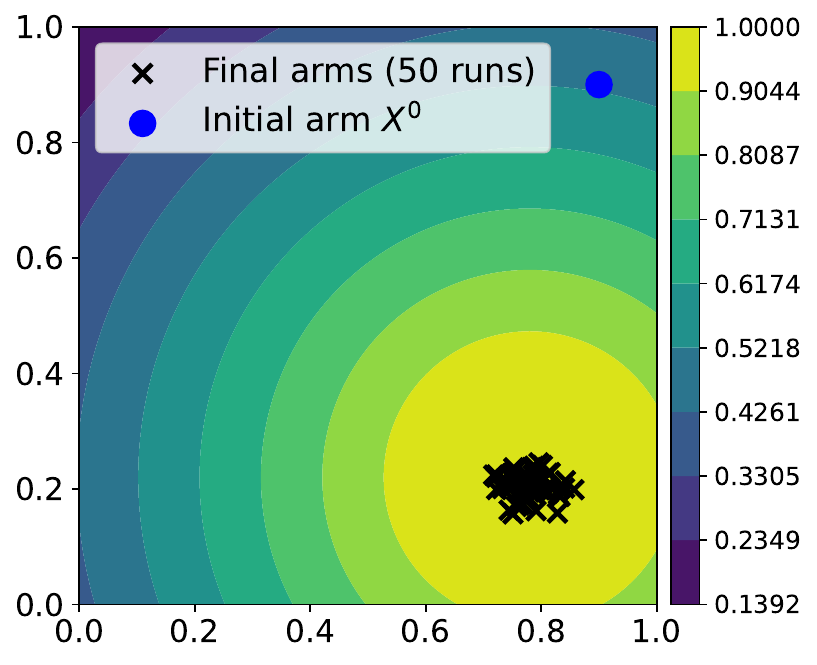}
        \caption{Reference arm convergence}
        \label{fig:flat_ref_conv}
    \end{subfigure}
    \hfill
    \begin{subfigure}[t]{0.32\linewidth}
        \centering
        \includegraphics[width=\linewidth]{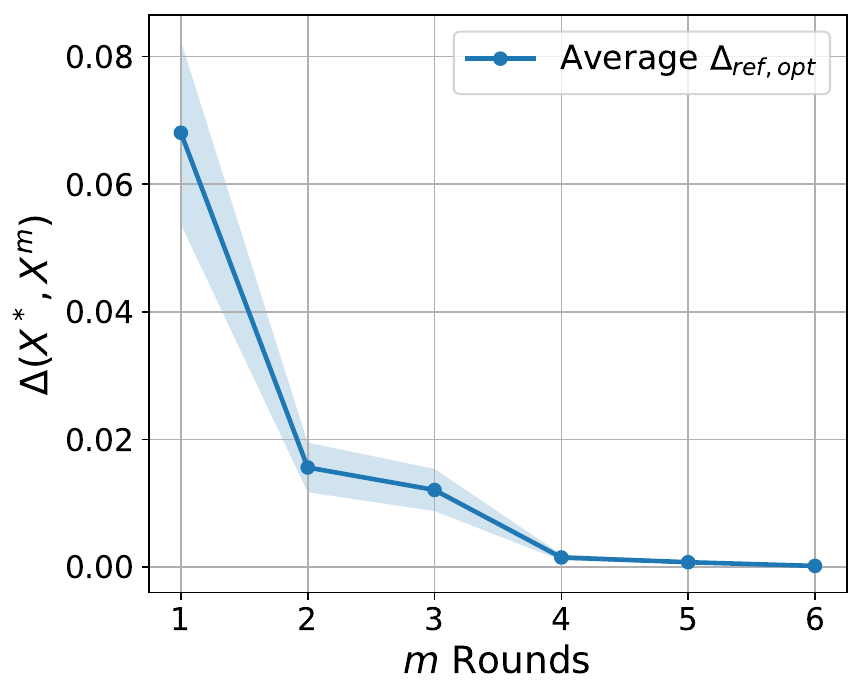}
        \caption{Convergence of $\Delta(x^*,X^m)$}
        \label{fig:flat_delta_conv}
    \end{subfigure}
    \hfill
    \begin{subfigure}[t]{0.32\linewidth}
        \centering
        \includegraphics[width=\linewidth]{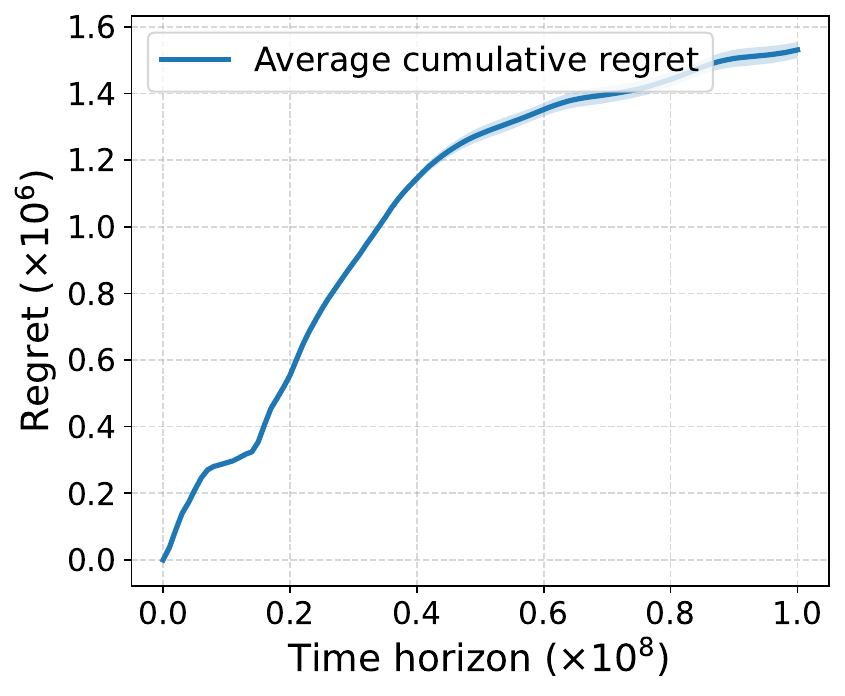}
        \caption{Cumulative Regret}
        \label{fig:flat_regret}
    \end{subfigure}

    \caption{Performance of \ouralgo\ on the Flat reward function, averaged over $50$ independent runs. From left to right: (a) evolution of the reference arm overlaid on the utility landscape, illustrating convergence toward the optimal region, (b) convergence of the optimality gap $\Delta(x^*,X^m)$ with the shaded region denoting one standard deviation, and (c) cumulative regret as a function of the number of comparisons.}

    \label{fig:flat_results}
\end{figure}

\paragraph{Plateau ($d_z=2$).}
Unlike standard optimization benchmarks that possess a unique maximizer, the Plateau reward contains an entire square region of optimal actions. Since every point inside the plateau attains the maximum reward, the algorithm must identify the optimal region rather than converge to a single point. This benchmark evaluates whether the adaptive partitioning strategy can efficiently localize a flat optimum.

The utility function is
\[
\mu(x)=
\operatorname{clip}\!\left(
1-\operatorname{dist}(x,[0.35,0.65]^2),
0,1
\right),
\]
where $\operatorname{dist}(x,S)$ denotes the Euclidean distance from $x$ to the set $S$.

The corresponding experimental results are presented in Fig.~\ref{fig:plateau_results}.

\begin{figure}[h]
    \centering

    \begin{subfigure}[t]{0.32\linewidth}
        \centering
        \includegraphics[width=\linewidth]{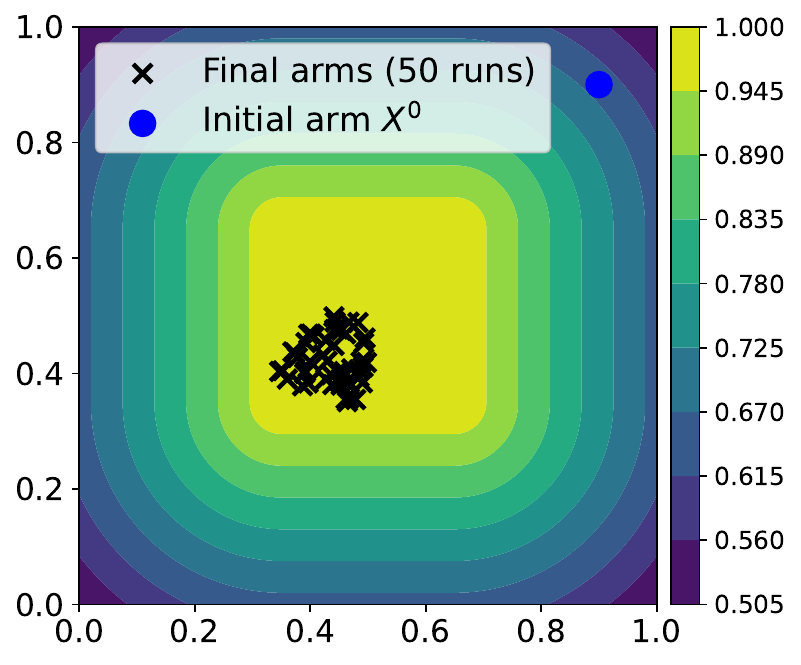}
        \caption{Reference arm convergence}
        \label{fig:plateau_ref_conv}
    \end{subfigure}
    \hfill
    \begin{subfigure}[t]{0.32\linewidth}
        \centering
        \includegraphics[width=\linewidth]{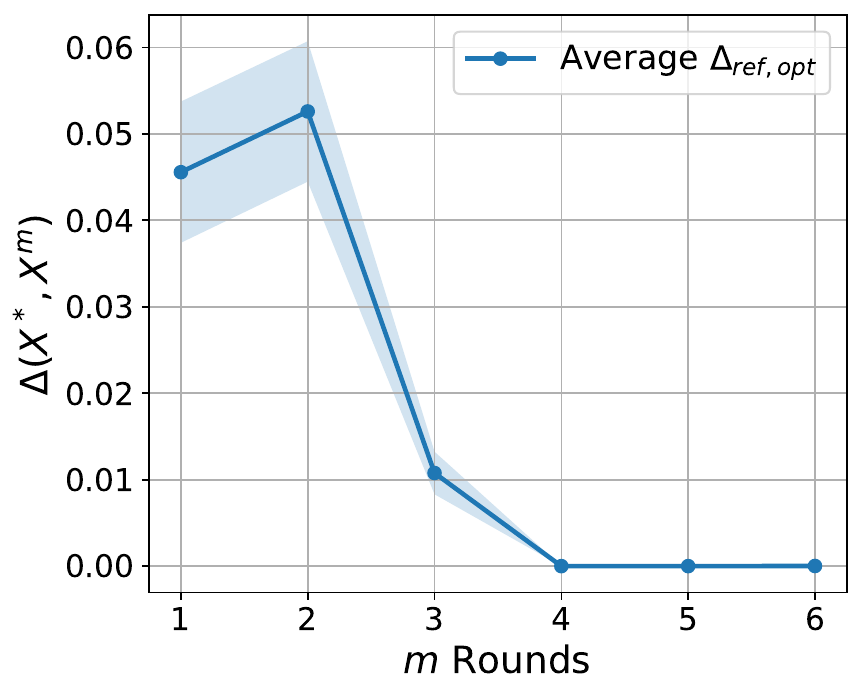}
        \caption{Convergence of $\Delta(x^*,X^m)$}
        \label{fig:plateau_delta_conv}
    \end{subfigure}
    \hfill
    \begin{subfigure}[t]{0.32\linewidth}
        \centering
        \includegraphics[width=\linewidth]{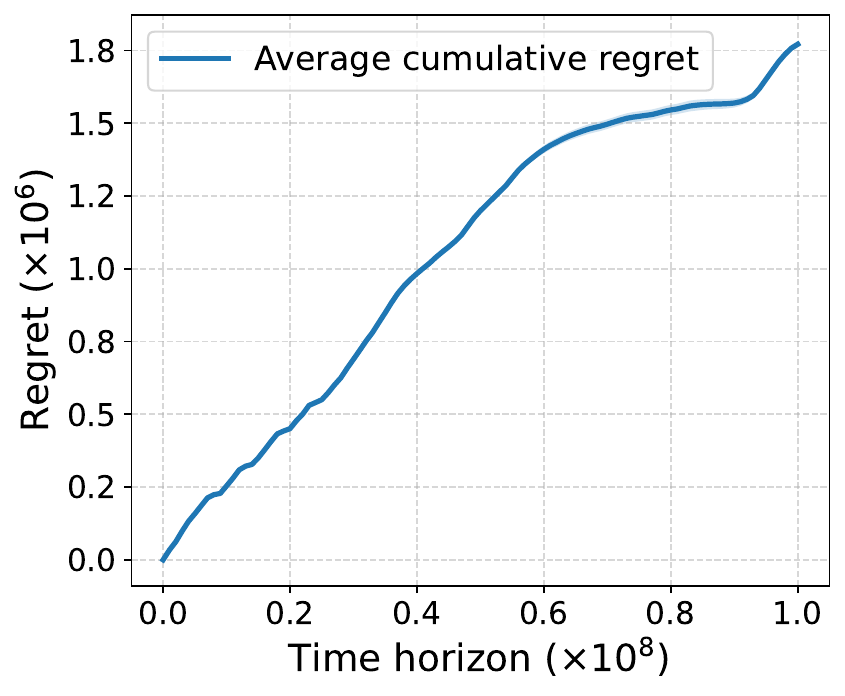}
        \caption{Cumulative Regret}
        \label{fig:plateau_regret}
    \end{subfigure}

   \caption{Performance of \ouralgo\ on the Plateau reward function, averaged over $50$ independent runs. From left to right: (a) evolution of the reference arm on the utility landscape, (b) convergence of the optimality gap $\Delta(x^*,X^m)$ with one standard deviation, and (c) cumulative regret over time.}

    \label{fig:plateau_results}
\end{figure}

\paragraph{Quadratic ($d_z=1$).}
The Quadratic reward has a smooth quadratic basin around the optimum located at $(0.5,0.5)$, followed by linear decay outside a small neighborhood. This function represents a well-conditioned optimization problem and serves as an intermediate benchmark between the Sharp and Flat landscapes.

Let
\[
d(x)=\|x-(0.5,0.5)\|_2,\qquad r=0.12,
\]
and define
\[
g(x)=
\begin{cases}
\dfrac{d(x)^2}{2r}, & d(x)\le r,\\[1mm]
d(x)-\dfrac r2, & \text{otherwise}.
\end{cases}
\]
The utility is
\[
\mu(x)=\operatorname{clip}(1-g(x),0,1).
\]

The experimental performance is illustrated in Fig.~\ref{fig:quadratic_results}.

\begin{figure}[h]
    \centering

    \begin{subfigure}[t]{0.32\linewidth}
        \centering
        \includegraphics[width=\linewidth]{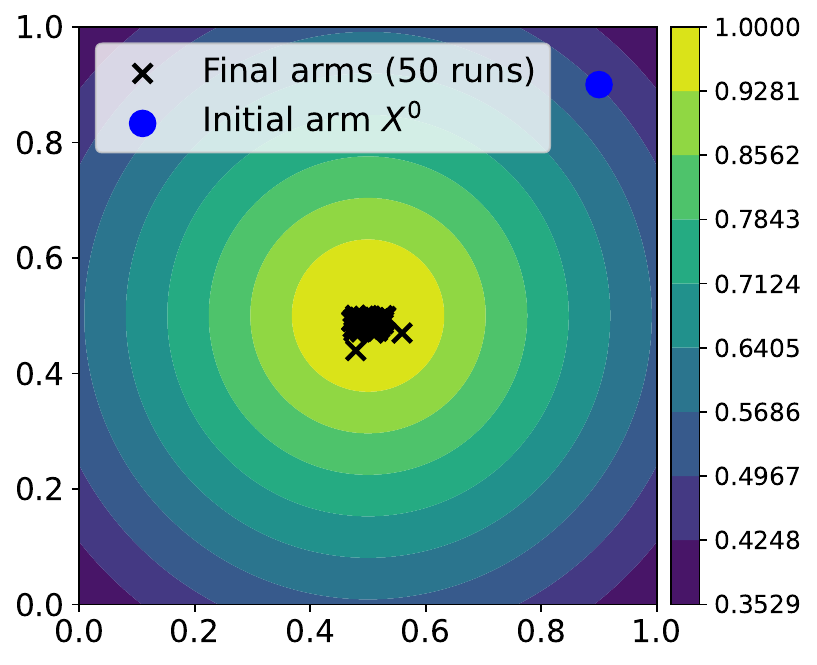}
        \caption{Reference arm convergence}
        \label{fig:quadratic_ref_conv}
    \end{subfigure}
    \hfill
    \begin{subfigure}[t]{0.32\linewidth}
        \centering
        \includegraphics[width=\linewidth]{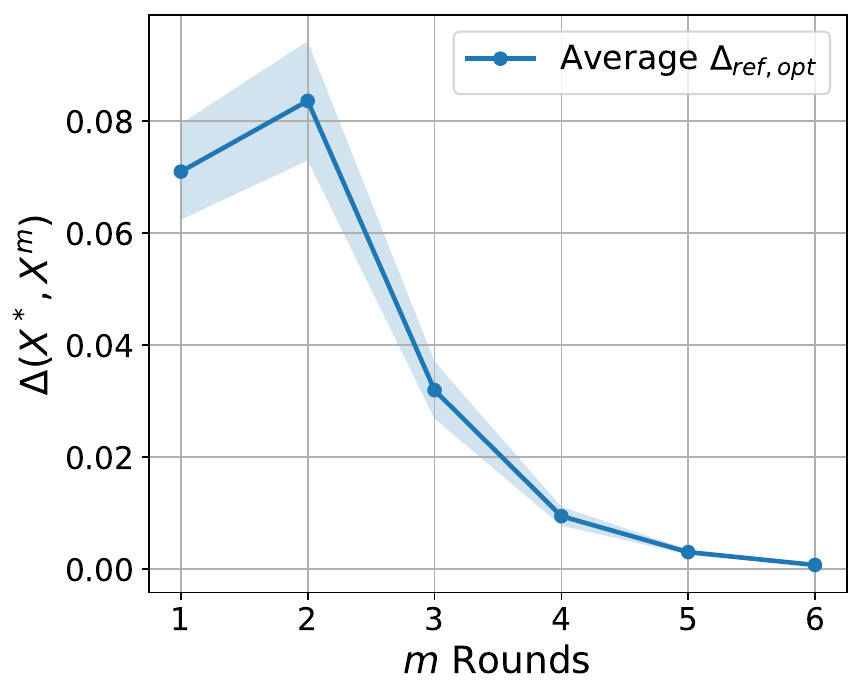}
        \caption{Convergence of $\Delta(x^*,X^m)$}
        \label{fig:quadratic_delta_conv}
    \end{subfigure}
    \hfill
    \begin{subfigure}[t]{0.32\linewidth}
        \centering
        \includegraphics[width=\linewidth]{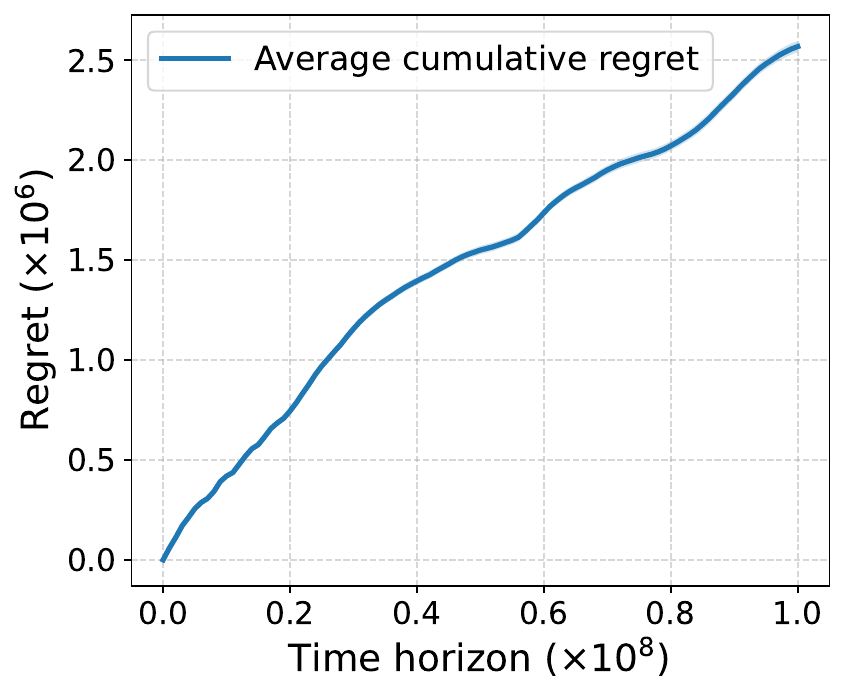}
        \caption{Cumulative Regret}
        \label{fig:quadratic_regret}
    \end{subfigure}

    \caption{Performance of \ouralgo\ on the Quadratic reward function, averaged over $50$ independent runs. From left to right: (a) convergence of the reference arm on the utility landscape, (b) convergence of the optimality gap $\Delta(x^*,X^m)$ with one standard deviation, and (c) cumulative regret as the learning process progresses.}

    \label{fig:quadratic_results}
\end{figure}

\paragraph{Ridge ($d_z=1$).}
The Ridge reward is fundamentally different from unimodal objectives because its optimum forms a continuous circular ring rather than a single point. Consequently, the learner only needs to identify the optimal manifold instead of a unique maximizer. This benchmark evaluates the algorithm's ability to adapt to non-isolated optima.

The utility function is
\[
\mu(x)=
\operatorname{clip}
\!\left(
1-\left|\|x-(0.5,0.5)\|_2-0.28\right|,
0,1
\right).
\]

The experimental results are reported in Fig.~\ref{fig:Ridge_results}.

\begin{figure}[h]
    \centering

    \begin{subfigure}[t]{0.32\linewidth}
        \centering
        \includegraphics[width=\linewidth]{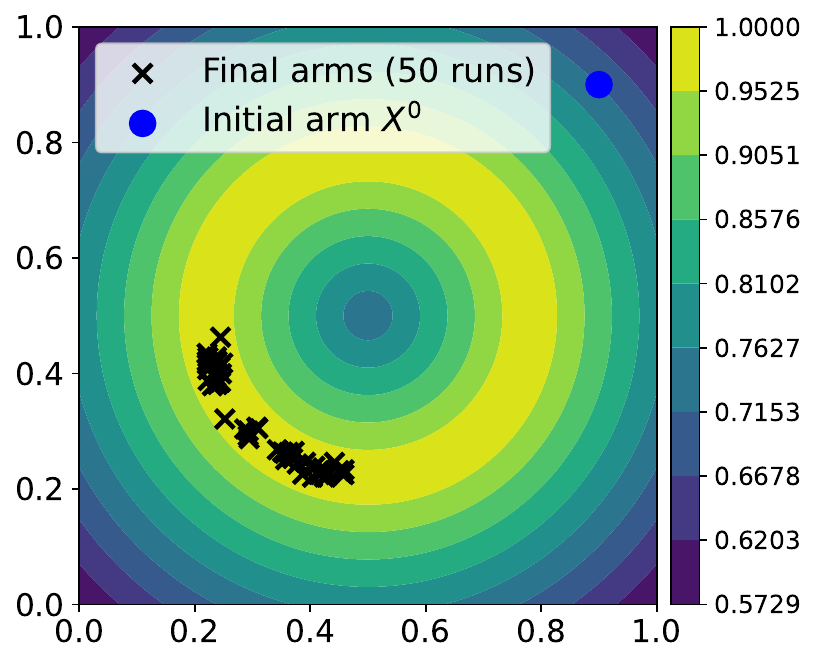}
        \caption{Reference arm convergence}
        \label{fig:Ridge_ref_conv}
    \end{subfigure}
    \hfill
    \begin{subfigure}[t]{0.32\linewidth}
        \centering
        \includegraphics[width=\linewidth]{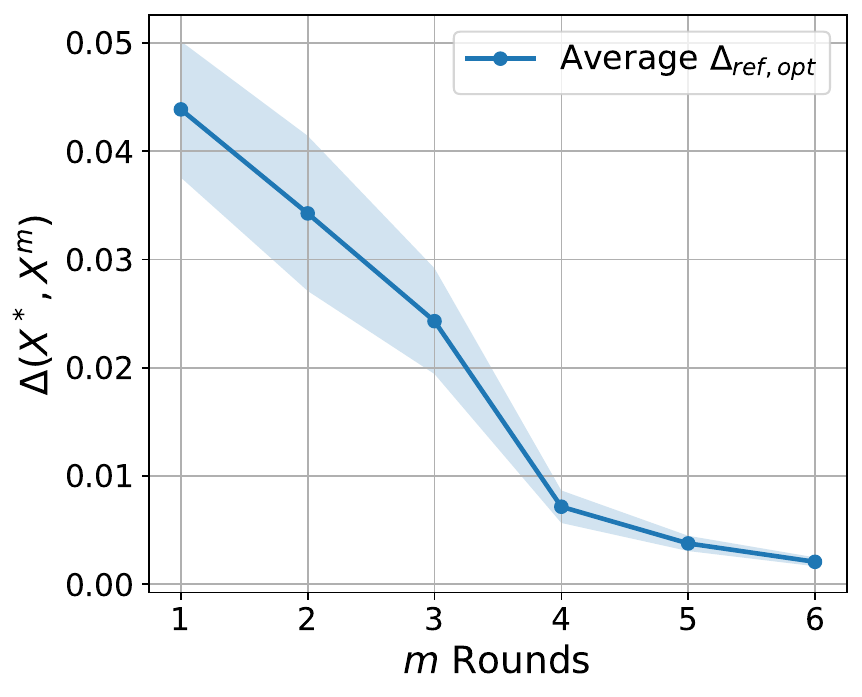}
        \caption{Convergence of $\Delta(x^*,X^m)$}
        \label{fig:Ridge_delta_conv}
    \end{subfigure}
    \hfill
    \begin{subfigure}[t]{0.32\linewidth}
        \centering
        \includegraphics[width=\linewidth]{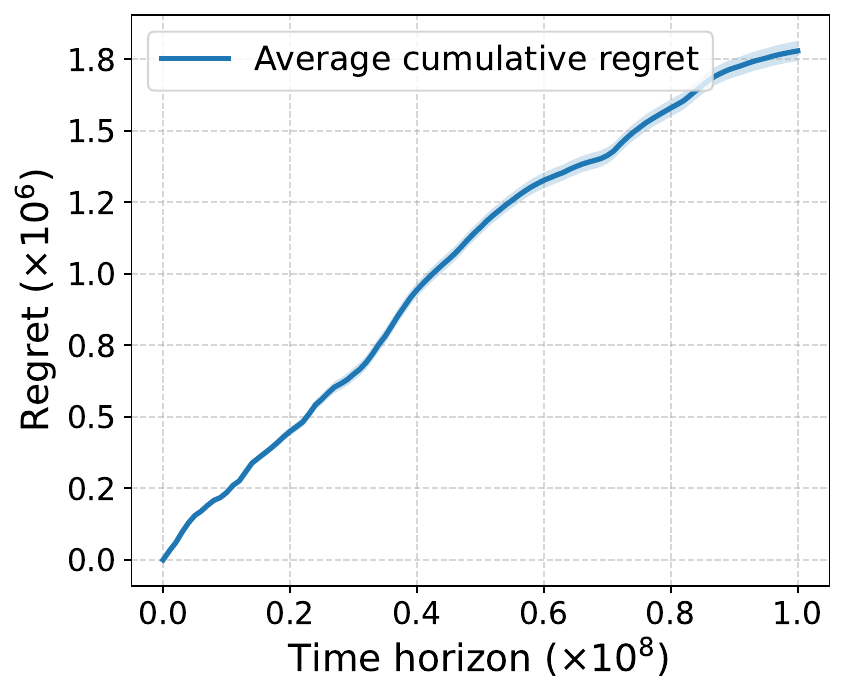}
        \caption{Cumulative Regret}
        \label{fig:Ridge_regret}
    \end{subfigure}

    \caption{Performance of \ouralgo\ on the Ridge reward function, averaged over $50$ independent runs. From left to right: (a) convergence of the reference arm toward the optimal ridge, (b) convergence of the optimality gap $\Delta(x^*,X^m)$ with one standard deviation, and (c) cumulative regret over time.}

    \label{fig:Ridge_results}
\end{figure}

\paragraph{Sharp ($d_z=0$).}
The Sharp reward is the simplest benchmark considered in our experiments. Its utility decreases linearly from a unique optimum located at $(0.5,0.5)$, producing large reward gaps between neighboring actions. Owing to its steep landscape, the optimal arm can typically be identified with relatively few comparisons.

The utility function is
\[
\mu(x)=
\operatorname{clip}
\!\left(
1-\|x-(0.5,0.5)\|_2,
0,1
\right).
\]

Figure~\ref{fig:sharp_results} summarizes the corresponding experimental results.

\begin{figure}[h]
    \centering

    \begin{subfigure}[t]{0.32\linewidth}
        \centering
        \includegraphics[width=\linewidth]{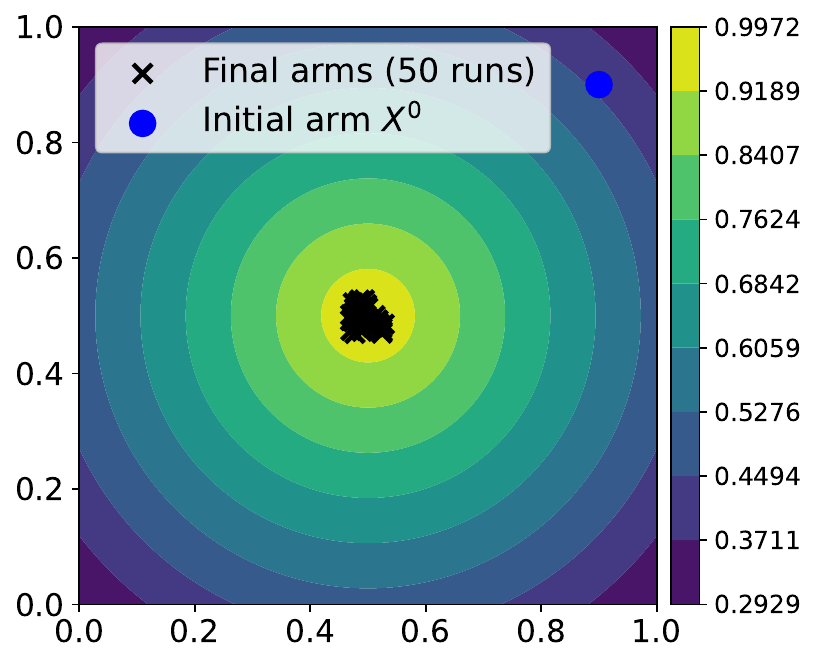}
        \caption{Reference arm convergence}
        \label{fig:sharp_ref_conv}
    \end{subfigure}
    \hfill
    \begin{subfigure}[t]{0.32\linewidth}
        \centering
        \includegraphics[width=\linewidth]{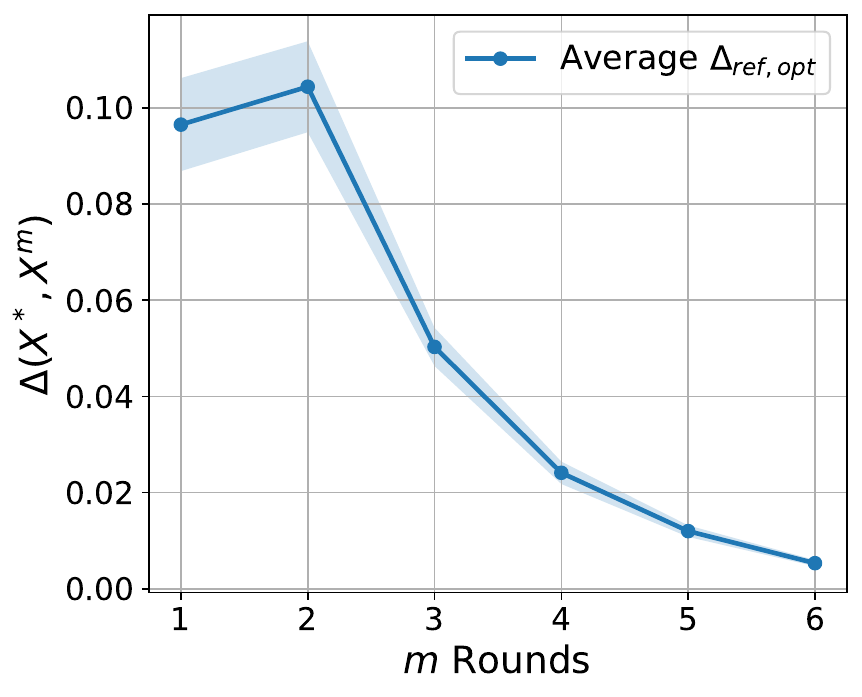}
        \caption{Convergence of $\Delta(x^*,X^m)$}
        \label{fig:sharp_delta_conv}
    \end{subfigure}
    \hfill
    \begin{subfigure}[t]{0.32\linewidth}
        \centering
        \includegraphics[width=\linewidth]{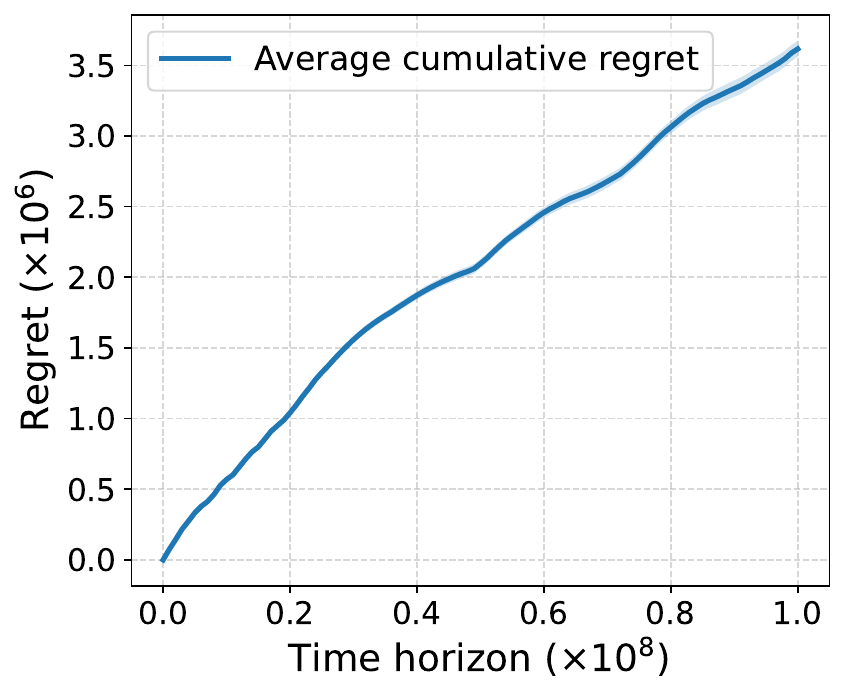}
        \caption{Cumulative Regret}
        \label{fig:sharp_regret}
    \end{subfigure}

    \caption{Performance of \ouralgo\ on the Sharp reward function, averaged over $50$ independent runs. From left to right: (a) convergence of the reference arm on the utility landscape, (b) convergence of the optimality gap $\Delta(x^*,X^m)$ with one standard deviation, and (c) cumulative regret as a function of the number of comparisons.}

    \label{fig:sharp_results}
\end{figure}

\paragraph{Two Peaks ($d_z=0$).}
The Two Peaks reward introduces two competing maxima with slightly different heights. While the global optimum lies at $(0.2,0.2)$, the second peak at $(0.8,0.75)$ is nearly optimal, making exploration more difficult since the algorithm must discriminate between two highly competitive regions before identifying the true optimum.

Define
\[
g_1(x)=1-\frac{\|x-(0.2,0.2)\|_2}{\sqrt2},
\qquad
g_2(x)=0.97-\frac{\|x-(0.8,0.75)\|_2}{\sqrt2}.
\]
The utility is
\[
\mu(x)=
\operatorname{clip}
\!\left(
\max\{g_1(x),g_2(x)\},
0,1
\right).
\]

The corresponding experimental results are shown in Fig.~\ref{fig:two_peaks_results}.

\begin{figure}[h]
    \centering

    \begin{subfigure}[t]{0.32\linewidth}
        \centering
        \includegraphics[width=\linewidth]{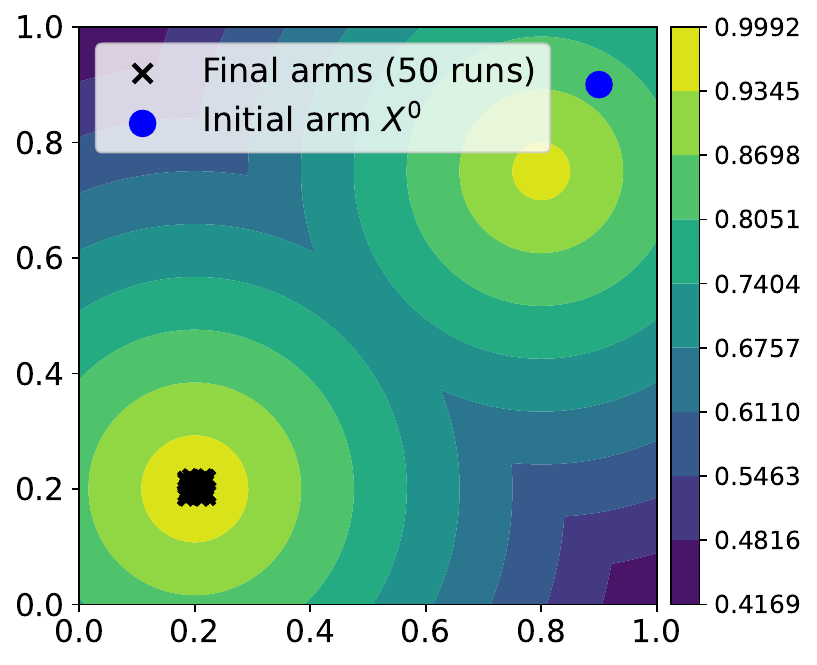}
        \caption{Reference arm convergence}
        \label{fig:two_peaks_ref_conv}
    \end{subfigure}
    \hfill
    \begin{subfigure}[t]{0.32\linewidth}
        \centering
        \includegraphics[width=\linewidth]{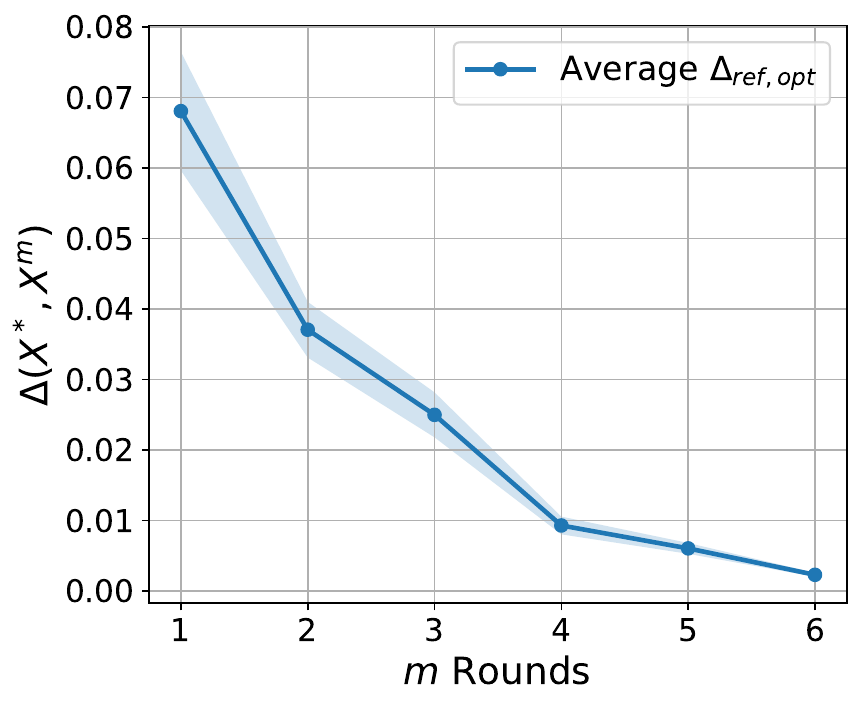}
        \caption{Convergence of $\Delta(x^*,X^m)$}
        \label{fig:two_peaks_delta_conv}
    \end{subfigure}
    \hfill
    \begin{subfigure}[t]{0.32\linewidth}
        \centering
        \includegraphics[width=\linewidth]{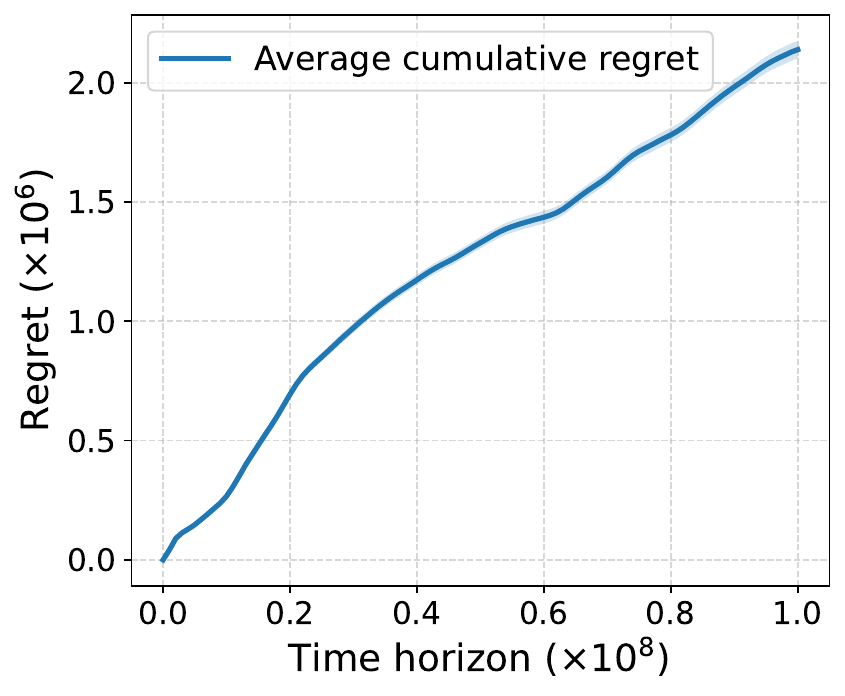}
        \caption{Cumulative Regret}
        \label{fig:two_peaks_regret}
    \end{subfigure}

    \caption{Performance of \ouralgo\ on the Two Peaks reward function, averaged over $50$ independent runs. From left to right: (a) convergence of the reference arm toward the global optimum despite the presence of a competing local optimum, (b) convergence of the optimality gap $\Delta(x^*,X^m)$ with one standard deviation, and (c) cumulative regret over time.}

    \label{fig:two_peaks_results}
\end{figure}

\subsection{Experiments with different transfer function}

To examine the robustness of our algorithm to the choice of the
transfer function, we additionally consider the sigmoid, arctangent,
probit, linear, and hyperbolic-tangent links. For a utility difference
\(u=f(x)-f(y)\), these transfer functions are defined as
\[
\begin{gathered}
\rho_{\mathrm{sig}}(u)
=\frac{1}{1+\exp(-u)},\qquad
\rho_{\mathrm{arc}}(u)
=\frac{1}{2}+\frac{1}{\pi}\arctan(u),\qquad
\rho_{\mathrm{probit}}(u)
=\Phi(u),\\[2mm]
\rho_{\mathrm{lin}}(u)
=\frac{1+u}{2},\qquad
\rho_{\tanh}(u)
=\frac{1+\tanh(u)}{2}.
\end{gathered}
\]where \(\Phi\) denotes the cumulative distribution function of the
standard normal distribution. All these functions satisfy
\(\rho(0)=1/2\) and
\(\rho(-u)=1-\rho(u)\). Their Lipschitz constants are obtained by
bounding their derivatives:
\[
\Gamma_{\mathrm{sig}}=\frac14,\qquad
\Gamma_{\mathrm{arc}}=\frac1\pi,\qquad
\Gamma_{\mathrm{probit}}=\frac1{\sqrt{2\pi}},\qquad
\Gamma_{\mathrm{lin}}=\frac12,\qquad
\Gamma_{\tanh}=\frac12.
\]
Figure~\ref{fig:different_tranfer_function} shows that the reference arm converges
towards the near-optimal region for all five transfer functions, while
the preference gap \(\Delta(x^\star,X^m)\) decreases consistently
across rounds. The cumulative-regret curves exhibit the same overall
sublinear behavior, although their numerical values differ because
the transfer functions have different slopes and therefore induce
different preference gaps for the same utility difference. These
results indicate that the empirical behavior of the algorithm is
robust to the particular transfer function used to generate the
pairwise comparisons.

\begin{figure}[h]
    \centering

    \begin{subfigure}[t]{0.32\linewidth}
        \centering
        \includegraphics[width=\linewidth]{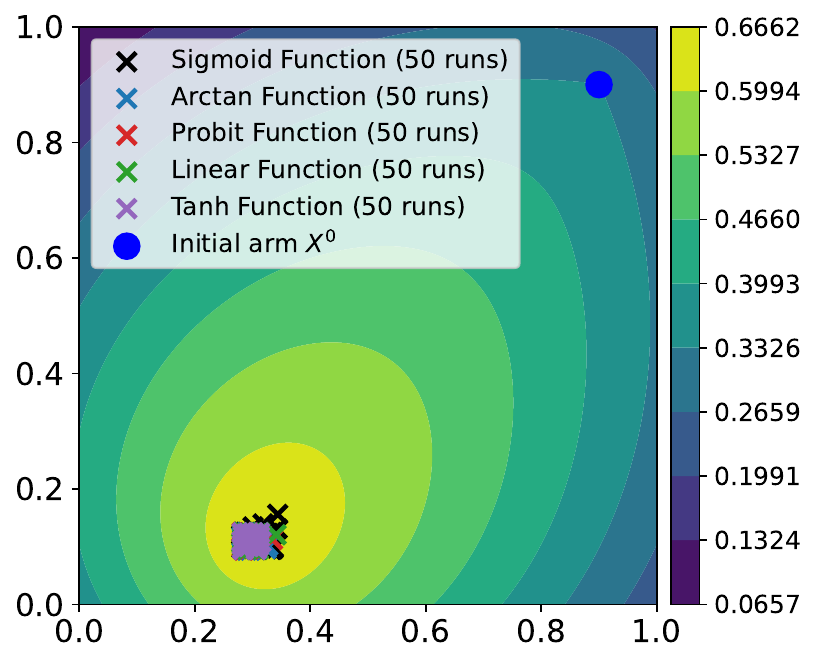}
        \caption{Reference arm convergence}
        \label{fig:sharp_ref_conv}
    \end{subfigure}
    \hfill
    \begin{subfigure}[t]{0.32\linewidth}
        \centering
        \includegraphics[width=\linewidth]{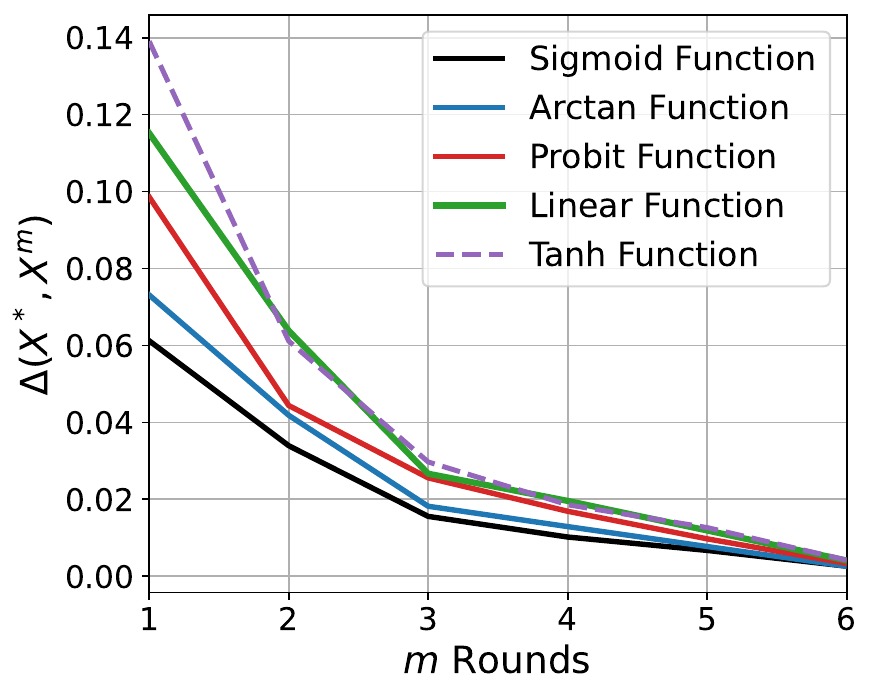}
        \caption{Convergence of $\Delta(x^*,X^m)$}
        \label{fig:sharp_delta_conv}
    \end{subfigure}
    \hfill
    \begin{subfigure}[t]{0.32\linewidth}
        \centering
        \includegraphics[width=\linewidth]{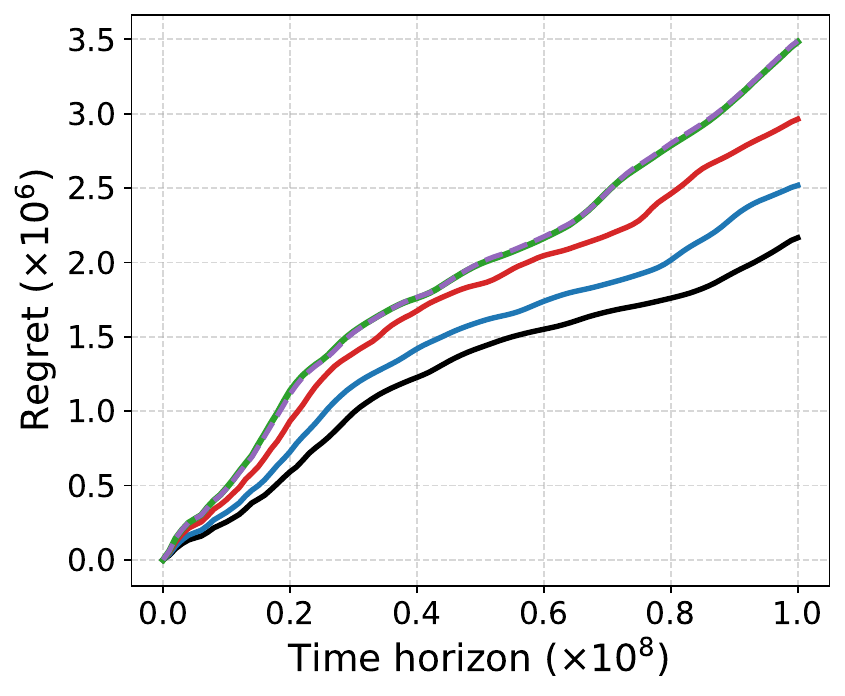}
        \caption{Cumulative Regret}
        \label{fig:sharp_regret}
    \end{subfigure}

    \caption{Performance of \ouralgo\ under five transfer functions—sigmoid, arctangent, probit, linear, and hyperbolic tangent—averaged over $50$ independent runs. From left to right: (a) convergence of the reference arm, (b) convergence of the preference gap $\Delta(x^\star,X^m)$ with one standard deviation, and (c) cumulative regret over time.
.}

    \label{fig:different_tranfer_function}
\end{figure}

\subsection{Experiments with different Ambient Dimension $d$ }

We further evaluate the performance of our algorithm across different ambient dimensions. In particular, we consider the cases of $D=3$ and $D=4$, in addition to the two-dimensional setting. For both settings, we use the utility function
$\mu(x)=1-\frac{2}{3}|x-x_1|-\frac{1}{3}|x-x_2|, $
where the locations of $x_1$ and $x_2$ are specified according to the ambient dimension. For $D=3$, we consider the action space $[0,1]^3$ with
$x_1=(0.3,0.1,0.5), x_2=(0.9,0.9,0.5),$
while for $D=4$, we consider the action space $[0,1]^4$ with
$ x_1=(0.3,0.1,0.5,0.5), x_2=(0.9,0.9,0.5,0.5). $ 
For these higher-dimensional settings, the algorithm terminates in lesser rounds. This behavior is primarily due to the substantially larger number of cubes generated during the partitioning process as the ambient dimension increases. Consequently, each round involves a larger number of pairwise comparisons than in the $D=2$ setting, resulting in a higher computational cost per round.

Despite this increased computational burden, the results in Figure~\ref{fig:transfer_delta_higher_dim} demonstrate that the reference arm continues to move progressively closer to the optimal arm in both $D=3$ and $D=4$. This indicates that the algorithm maintains its ability to identify increasingly promising regions of the action space even in higher dimensions. Although the available computational budget of $10^8$ limits the number of completed rounds, the observed convergence trend suggests that, with a sufficiently larger time horizon $T$, the reference arm would move even closer to the optimal arm and ultimately approach the true optimal solution.

\begin{figure}[h]
    \centering

    \begin{subfigure}[t]{0.32\linewidth}
        \centering
        \includegraphics[width=\linewidth]{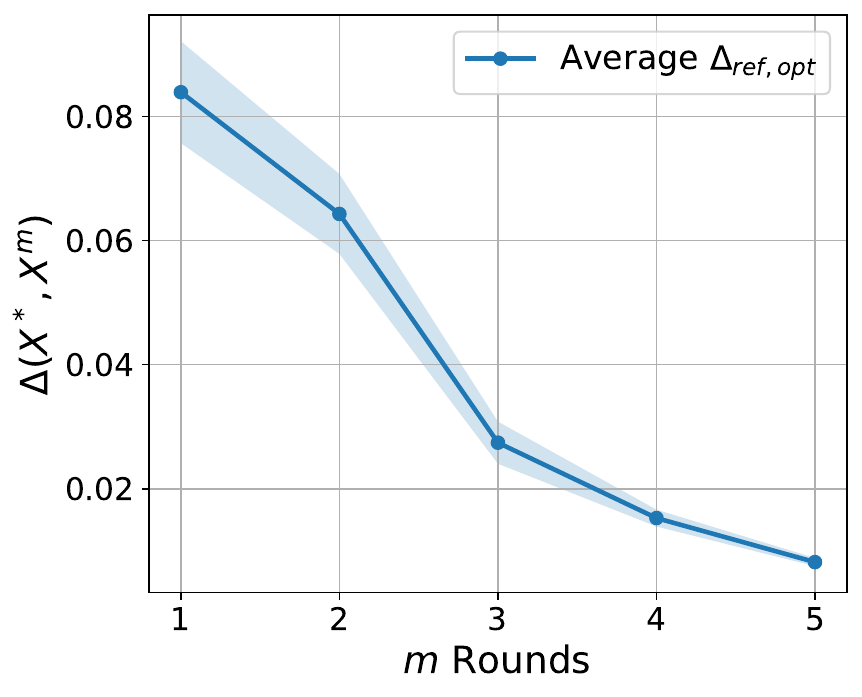}
        \caption{$D=3$}
        \label{fig:transfer_delta_D3}
    \end{subfigure}%
    \begin{subfigure}[t]{0.32\linewidth}
        \centering
        \includegraphics[width=\linewidth]{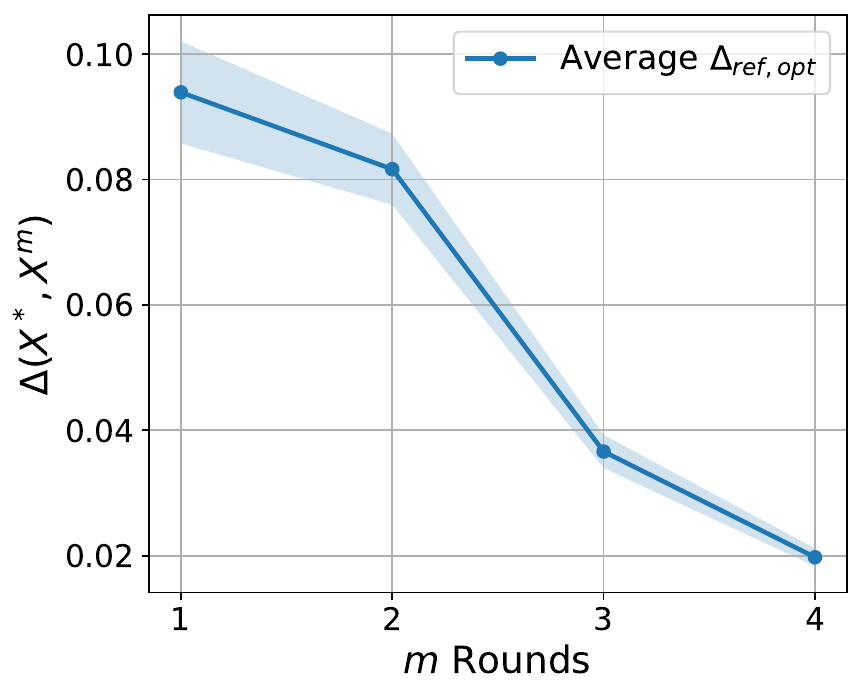}
        \caption{$D=4$}
        \label{fig:transfer_delta_D4}
    \end{subfigure}

    \caption{
    Convergence of the preference gap $\Delta(x^\star,X^m)$
    for higher-dimensional action spaces under \ouralgo.
    Panel~(a) shows the convergence of $\Delta(x^\star,X^m)$
    for an ambient dimension of $D=3$, while Panel~(b) shows
    the corresponding convergence for $D=4$. In both settings,
    the preference gap decreases with the number of rounds,
    indicating that the reference arm progressively approaches
    the optimal arm.
    }
    \label{fig:transfer_delta_higher_dim}
\end{figure}

\clearpage

\end{document}